%% file: example_paper.tex
\documentclass{article}

\usepackage{microtype}
\usepackage{graphicx}
\usepackage{subfigure}
\usepackage{booktabs} 

\usepackage{macro}

\usepackage{hyperref}

\usepackage[accepted]{icml2024}

\usepackage{amsmath}
\usepackage{amssymb}
\usepackage{mathtools}
\usepackage{amsthm}

\usepackage[capitalize,noabbrev]{cleveref}

\hypersetup{
    unicode=false,          
    colorlinks=true,        
    linkcolor=red,          
    citecolor=darkgreen,    
    filecolor=magenta,      
    urlcolor=blue           
}

\usepackage{thm-restate}
\theoremstyle{plain}
\newtheorem{theorem}{Theorem}[section]
\newtheorem{proposition}[theorem]{Proposition}
\newtheorem{lemma}[theorem]{Lemma}
\newtheorem{corollary}[theorem]{Corollary}
\theoremstyle{definition}
\newtheorem{definition}[theorem]{Definition}

\theoremstyle{remark}

\usepackage[textsize=tiny]{todonotes}

\icmltitlerunning{Causal Discovery with Fewer Conditional Independence Tests}

\begin{document}

\twocolumn[
\icmltitle{Causal Discovery with Fewer Conditional Independence Tests}



\icmlsetsymbol{equal}{*}

\begin{icmlauthorlist}
\icmlauthor{Kirankumar Shiragur}{equal,yyy}
\icmlauthor{Jiaqi Zhang}{equal,yyy,comp}
\icmlauthor{Caroline Uhler}{comp}
\end{icmlauthorlist}

\icmlaffiliation{yyy}{Eric and Wendy Schmidt Center, Broad Institute}
\icmlaffiliation{comp}{Laboratory for Information \& Decision Systems, Massachusetts Institute of Technology}

\icmlcorrespondingauthor{Kirankumar Shiragur}{kshiragu@broadinstitute.org}
\icmlcorrespondingauthor{Jiaqi Zhang}{viczhang@mit.edu}

\icmlkeywords{Machine Learning, ICML}

\vskip 0.3in
]



\printAffiliationsAndNotice{\icmlEqualContribution} 

\input{content/abstract}

\input{content/introduction}

\input{content/preliminaries}
\input{content/results}

\input{content/prefix_subset_alg}

\input{content/TCG_alg}

\input{content/interventions}

\input{content/experiment}
\input{content/discussion}

\section*{Acknowledgements}
We thank the anonymous reviewers for helpful feedback. J.Z.~was partially supported by an Apple AI/ML PhD Fellowship.
K.S.~ was supported by a fellowship from the Eric and Wendy Schmidt Center at the Broad Institute. C.U.~was partially supported by NCCIH/NIH (1DP2AT012345), ONR (N00014-22-1-2116), DOE-ASCR (DE-SC0023187), the MIT-IBM Watson AI Lab, and a Simons Investigator Award.

\section*{Impact Statement}
This paper presents theoretical work whose goal is to advance the field of causal inference. There are many potential societal applications of our work, none of which we feel must be specifically highlighted here.

\bibliography{refs}
\bibliographystyle{icml2024}

\newpage
\appendix
\onecolumn
\input{appendix/supplement}

\end{document}

%% file: content/abstract.tex
\begin{abstract}

Many questions in science center around the fundamental problem of understanding causal relationships.
However, most constraint-based causal discovery algorithms, including the well-celebrated PC algorithm, often incur an \emph{exponential} number of conditional independence (CI) tests, posing limitations in various applications.
Addressing this, our work focuses on  characterizing what can be learned about the underlying causal graph with a reduced number of CI tests. 
We show that it is possible to a learn a coarser representation of the hidden causal graph with a \emph{polynomial} number of tests. 
This coarser representation, named Causal Consistent Partition Graph (CCPG), comprises of a partition of the vertices and a directed graph defined over its components. 
CCPG satisfies consistency of orientations and additional constraints which favor finer partitions.
Furthermore, it reduces to the underlying causal graph when the causal graph is identifiable. 
As a consequence, our results offer the first efficient algorithm for recovering the true causal graph with a polynomial number of tests, in special cases where the causal graph is fully identifiable through observational data and potentially additional interventions.

\end{abstract}

%% file: content/introduction.tex
\section{Introduction}
\label{sec:introduction}

Causal discovery is a fundamental task in various scientific disciplines including biology, economics, and sociology \cite{king2004functional,cho2016reconstructing,tian2016bayesian, sverchkov2017review,rotmensch2017learning,pingault2018using, de2019combining, reichenbach1956direction,woodward2005making,eberhardt2007interventions, hoover1990logic, friedman2000using,robins2000marginal, spirtes2000causation, pearl2003causality}. 
Directed acyclic graphs (DAGs) stand out as a popular choice for representing causal relations, with edge directions signifying the flow of information between variables. The core objective of causal discovery is to identify both the edges and their orientations based on available data. While certain structures can be recovered from observational data \cite{verma1990equivalence}, orienting the full graph often requires additional experiments or interventions.

Research on causal structure learning from observational data dates back to the 1990s \cite{verma1990equivalence, spirtes1989causality}. As a pioneering work in this direction, the PC algorithm \cite{spirtes2000causation}, named after the authors Peter Spirtes and Clark Glymour, still remains one of most popular and widely used algorithms. It recovers the structure using observational data through conditional independence (CI) tests, with the number of tests being exponential in the degree of the graph. 
Following this, many causal discovery algorithms emerged \cite{kalisch2007estimating, brenner2013sparsityboost, alonso2013scaling, schulte2010imap}, accommodating diverse and more general settings, including the presence of latent variables \cite{spirtes1999algorithm, spirtes2013causal} and interventional data \cite{ eberhardt2005number,eberhardt2006n}. However, a common challenge shared by these algorithms is their reliance, to different extents, on an \emph{exponential} number of CI tests in certain graph parameters. This inherent dependence on an exponential number of tests poses practical challenges, making them unsuitable for many real-world scenarios. Moreover, it suggests that achieving \emph{exact} causal structure learning can be highly challenging.

As performing exponential number of tests is limited in many applications, it motivates us to study the following question:
\begin{center}
\emph{What useful information about the underlying causal graph can be inferred with fewer conditional independence tests?}
\end{center}
Aligned with this motivation, our work also explores the role of interventions in the structure learning process. 

In our work, we study these questions under standard Markov, faithfulness and causal sufficiency assumptions \citep{lauritzen1996graphical,spirtes2000causation}. The primary contribution of our work is an efficient algorithm that uses a \emph{polynomial} number of CI tests and recovers a representation of the underlying causal graph with observational and optionally interventional datasets. This representation consists of a partition of the vertices and a DAG defined over its components which is consistent with the underlying causal graph.
In addition, our representation is designed to avoid dummy partitions that group all the vertices into a single component. The definition of our representation ensures that the components in the partition satisfy several additional properties, guaranteeing that each component either contains a single vertex or comprises an edge that could only be oriented after an intervention is performed on one of its endpoints. We refer to this representation as the Causally Consistent Partition Graph (CCPG) representation.

An important implication of our results is that if the underlying causal graph is fully identifiable using only observational data, our algorithm yields a CCPG with a partition containing components, each of which is of size one. The size-one property of each component means that our algorithm recovers the true causal graph using only a \emph{polynomial} number of conditional independence tests. We extend this result in the presence of interventions and provide an algorithm that recovers the true causal graph, when the set of interventions provided is sufficient to identify the underlying causal graph. To the best of our knowledge, our algorithms present the first to guarantee recovering the true causal graph using a polynomial number of tests when the graph is either entirely identifiable from observational data or with an additional set of interventions. 

\subsection{Related Works}




Efficient algorithms for causal structure learning \cite{spirtes2000causation, claassen2013learning} exist for constant bounded degree graphs, recovering the causal graph with a polynomial number of CI tests. 
For general causal graphs, current methods often entail an exponential number of CI tests, where \cite{xie2008recursive,zhang2024membership} aimed to to reduce such complexity.
For Bayesian network learning, finding a minimal Bayesian network is NP-hard, even with a constant-time CI oracle and nodes with at most $k \geq 3$ parents. \citet{chickering2004large} demonstrated this hardness through a polynomial reduction from the NP-complete problem, Degree-Bounded Feedback Arc Set. These findings highlight the contrast between causal structure and minimal Bayesian network learning, suggesting that causal structure learning is notably more straightforward. Our results further reinforce this notion by identifying a special class of causal graphs that can be recovered with a polynomial number of conditional independence tests. For other hardness results on Bayesian network learning, we refer readers to \citet{bouckaert1994properties, chickering2004large} and references therein.

Learning causal relationships from observational data \cite{verma1990equivalence, spirtes1989causality, spirtes2000causation, chickering2002optimal, geiger2002parameter, nandy2018high} and interventional data \cite{eberhardt2010causal,hu2014randomized,shanmugam2015learning,greenewald2019sample,squires2020active,choo2022verification,choo2023subset,shiragur2024meek} is a well studied problem with a rich literature. We encourage interested readers to explore \citet{glymour2019review,squires2022causal} and references therein for a more comprehensive understanding and further details.

\subsection{Organization} 
The rest of our paper is organized as follows. \Cref{sec:prelim} is our preliminary section. In \Cref{sec:results}, we provide all the main results of the paper. In \Cref{sec:4} and \Cref{sec:5} combined, we provide our CCPG recovery algorithm when just observational data is available. In \Cref{sec:6}, we extend our results to the case of interventions. We provide numerical results in \Cref{sec:exp}. Finally in \Cref{sec:7}, we conclude with a short discussion and few open directions.





%% file: content/preliminaries.tex
\section{Preliminaries}
\label{sec:prelim}

\newcommand{\src}{\mathrm{src}}

\subsection{Graph Definitions} 

Let $\cG$ be a directed acyclic graph (DAG) on $n$ vertices in $V$. 
For a vertex $v \in V$, let $\Pa(v), \Anc(v)$, and $\Des(v)$ denote the \emph{parents}, \emph{ancestors}, and \emph{descendants} of $v$ respectively. Let $\Anc[v] = \Anc(v) \cup \{v\}$ and $\Des[v] = \Des(v) \cup \{v\}$.

For a set of vertices $S\subset V$, denote $\Pa(S)=\cup_{v\in S} \Pa(s)$. Similarly define $\Anc(S), \Des(S), \Anc[S]$, and $\Des[S]$. 
We write $\src(S)$ as the set of \emph{source} nodes within $S$, that is, $$\src(S)=\{v \in S ~|~ \Anc(v) \cap S = \varnothing\}~.$$
Denote $\bar{S}=V\setminus S$. Let $\cG[S]$ be a \emph{subgraph} of $\cG$ by removing all vertices in $\bar{S}$.

A \emph{v-structure} refers to three distinct vertices $u,v,w$ such that $u \to v \gets w$ and $u,w$ are not adjacent.
An edge $u \to v$ is a \emph{covered edge} \cite{chickering2013transformational} if $\Pa[u]=\Pa(v)$. 
A \emph{path} in $\cG$ is a list of distinct vertices, where consecutive vertices are adjacent.
We can associate a \emph{topological ordering} $\pi : V \to [n]$ to any DAG $\cG$ such that any $u\to v$ in $\cG$ satisfy $\pi(u) < \pi(v)$. 
Note that such topological ordering is not necessarily unique.

\subsection{D-Separation and Conditional Independence}

DAGs are commonly used in causality \cite{pearl2009causality}, where vertices represent random variables and their joint distribution $P$ factorizes according to the DAG: $P(v_1, \dots, v_n) = \prod_{i=1}^n P(v_i \mid \Pa(v_i))$. This factorization entails a set of conditional independencies (CIs) in the \emph{observational} distribution $P$. These CI relations are fully characterized by \emph{d-separation} \cite{geiger1990logic}. Formally, for disjoint vertex sets $A,B,C\subset V$, sets $A,B$ are \emph{d-separated} by $C$ if and only if any path connecting $A$ and $B$ in $\cG$ is inactive given $C$. A path is \emph{inactive} given $C$ when it has a \emph{collider}\footnote{Vertex $d$ is a collider on a path iff $\cdot\to d\leftarrow\cdot$ on the path.} $d\not\in \Anc[C]$ or a non-collider $c\in C$; otherwise the path is \emph{active} given $C$. Figure~\ref{fig:d-sep} illustrates these concepts.

\begin{figure}[h]
    \centering
    \includegraphics[width=0.45\textwidth]{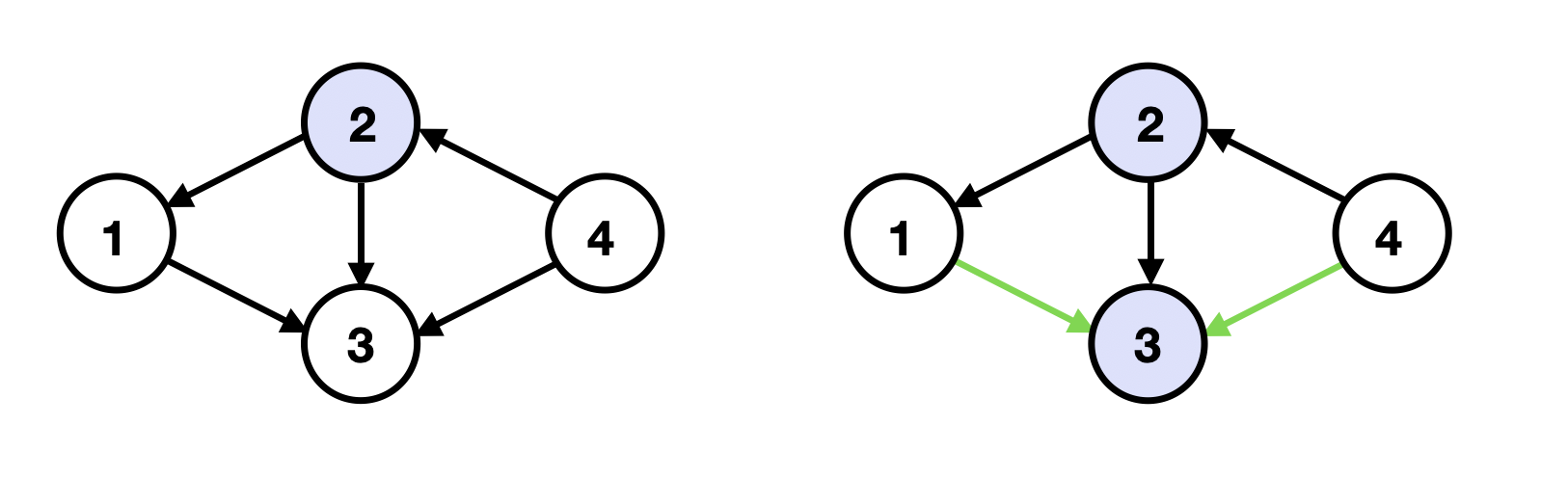}
    \caption{\textbf{(Left).} $\{1\}$ and $\{4\}$ are \emph{d-separated} by $\{2\}$, as all paths are \emph{inactive} given $\{2\}$. \textbf{(Right).} $\{1\}$ and $\{4\}$ are \emph{not} d-separated by $\{2,3\}$, as path $1\to 3\leftarrow 4$ is \emph{active} given $\{2,3\}$ by \emph{collider} $3$.}\label{fig:d-sep}
\end{figure}

We write $A\CI B\mid C$\footnote{For simplicity, we also write $A\CI B\mid C$ for potential overlapping sets to denote $A\CI B\mid C\setminus (A\cup B)$} when $A,B$ are conditionally independent given $C$ in the observational distribution $P$. If any set among $A,B,C$ contains only one node, e.g., $A=\{a\}$, we write $a\CI B\mid C$ for simplicity. When $C$ d-separates $A,B$, then it holds that $A\CI B\mid C$ (known as the global Markov property \cite{geiger1990logic}). Under the faithfulness assumption, the reverse also holds, i.e., all CI relations in $P$ are implied by d-separation in $\cG$.

\textbf{Setup.} In this work, we assume that the causal DAG $\cG$ is \emph{unknown}. But we assume causal sufficiency (i.e., no latent confounders), faithfulness and access to enough samples from $P$ to determine if $A\CI B\mid C$ for any $A,B,C\subset V$. As all CIs are implied by d-separations, we may infer information about $\cG$ using these tests.

\subsection{Interventions}

An \emph{intervention} $I\subset V$ is an experiment where the conditional distributions $P(v\mid \Pa(v))$ for $v\in I$ are changed into $P^I(v)$.\footnote{We consider hard interventions in this work.} Such interventions  eliminate the dependency between $v$ and $\Pa(v)$. Let $\cG^I$ denote the modified version of $\cG$, where all incoming edges to $v\in I$ are removed.
Let $P^I$ denote the interventional distribution, i.e., $P^I(v_1,\dots,v_n)=\prod_{v\in I} P^I(v)\prod_{v\not\in I}P(v\mid\Pa(v))$. Then $P^I$ factorizes with respect to $\cG^I$. We denote $A\CI_I B\mid C$ for CI tests in the interventional distribution $P^I$.

\textbf{Setup with Interventions.} Similar to the observational setting, we assume faithfulness of $P^I$ to $\cG^I$ and access to enough samples from $P^I$ to determine if $A\CI_I B\mid C$.

\subsection{Verifying Intervention Sets and Covered Edges}

When it is possible to perform \emph{any number of CI tests}: with observational data, a DAG $\cG$ is in general only identifiable up to its skeleton, v-structures 
\cite{andersson1997characterization}, and possibly additional edges given by the Meek rules \cite{meek1995}. Identifiability can be improved with interventional data \cite{hauser2012characterization}, where $I$ allows us to infer the edge orientation of any edge cut by $I$ and $V\setminus I$.

A \emph{verifying intervention set} $\cI$ for a DAG $\cG$ \cite{choo2022verification} is a set of interventions that fully orients $\cG$, possibly with repeated applications of the Meek rules. We will make use of the following result in our work.
\begin{proposition}[Theorem 9 in \cite{choo2022verification}]\label{prop:1}
   Set $\cI$ is a verifying intervention set if and only if for every covered edge $u\to v$ in $\cG$, there is $|I\cap\{u,v\}|=1$ for some $I\in\cI$. 
\end{proposition}

The \emph{verification number} $\nu(\cG)$ is defined as the minimum size of any verifying intervention set of $\cG$. This proposition tells us that $\nu(\cG)$ equals to the minimum size of any vertex cover of the covered edges in $\cG$. 

%% file: content/results.tex
\newcommand{\expt}[1]{\mathbb{E}\left[ #1\right]}
\newcommand{\findsrc}{\mathrm{FindSource}}
\newcommand{\Pt}{\mathcal{P}}
\newcommand{\iDes}[1]{\mathrm{iDes}(#1)}

\section{Main Results}
\label{sec:results}

Here we present our main findings. As highlighted in the introduction, the key contribution of our work lies in recovering a representation of the underlying causal graph that satisfies various desirable properties with very few CI tests. We now provide a formal definition of this representation.

\begin{definition}[CCPG \& $\cI$-CCPG]\label{def:TCG}
    A \emph{Causally Consistent Partition Graph (CCPG)} representation of a DAG $\cG$ on $V$ consists of a partition of $V$ into components $V_1,\dots,V_k$ and a DAG $\cD$ between the components such that,\\
    \textbf{(intra-component property):} for each $i\in [k]$, it holds that  $|\src(V_i)|=1$. Furthermore, if $|V_i|>1$, then $\cG[V_i]$ has at least one covered edge.\\
    \textbf{(inter-component property):} $\cD$ is topologically ordered, i.e., $i\rightarrow j$ in $\cD$ only if $V_i<V_j$. It is also consistent with $\cG$: (1) if there is no directed edge $i\rightarrow j$ in $\cD$, then there are no edges between $V_i$ and $V_j$ in $\cG$; (2) if there is a directed edge $i\to j$ in $\cD$, then there is $u\in V_i$ such that $u\in\Pa(\src(V_j))$. 

    We further define an \emph{Interventional Causally Consistent Partition Graph ($\cI$-CCPG)} representation of $\cG$ with respect to an intervention set $\cI$: an $\cI$-CCPG is a CCPG representation of $\cG$ that additionally satisfies the following \textbf{strong intra-component condition}: for each $i \in [k]$, if $|V_i|>1$ then $\cG[V_i]$ has at least one \emph{unintervened}\footnote{An edge is intervened by $I$ if only one of the vertices is in $I$.} covered edge.
\end{definition}

\begin{figure}[h]
    \centering
    \includegraphics[width=0.45\textwidth]{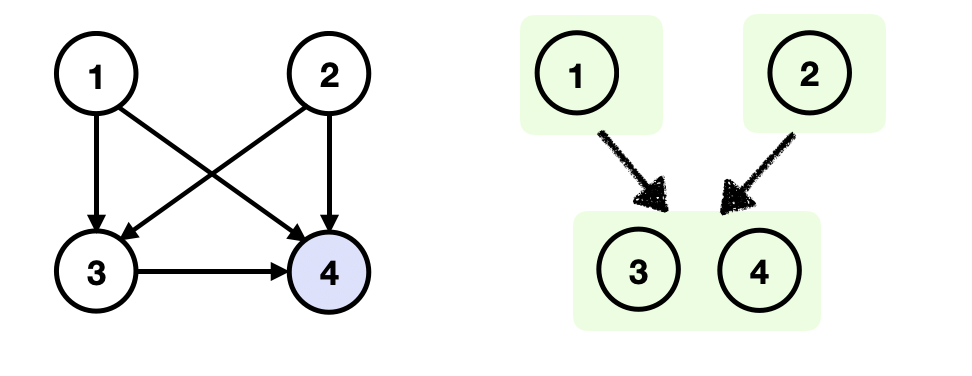}
    \caption{\textbf{Example of CCPG \& $\cI$-CCPG.} \textbf{(Left).} Ground-truth $\cG$. \textbf{(Right).} A CCPG representation of $\cG$, where $V_1,V_2,V_3$ are indicated by green boxes and $\cD$ is illustrated in chalk strokes. Vertices $3,4$ can be in one component as $3\to 4$ is a covered edge. For $\cI=\{4\}$, the only $\cI$-CCPG is $\cG$ itself (due to strong intra-component condition in \Cref{def:TCG}).
    }\label{fig:ccpg}
\end{figure}

Figure~\ref{fig:ccpg} illustrates these concepts. Note that when $\cI=\varnothing$, $\cI$-CCPG reduces to CCPG.

In Definition~\ref{def:TCG}, the first property prefers finer partitions, while the second property ensures consistency. Formally, we can show the following properties of these representations, which establish the significance of CCPGs. Proofs for all lemmas in this section can be found in Appendix~\ref{app:a}.

\begin{restatable}[Properties of CCPG]{lemma}{ccpgprop}\label{ccpg:prop}
    For any intervention set $\cI$ (including $\varnothing$), the following arguments hold:
    \begin{itemize}
        \item $\cD=\cG$ (i.e., partitioning $V$ into individual vertices) is a valid $\cI$-CCPG of $\cG$. 
        \item If the verification number of $\cG$ is zero, i.e., $\nu(\cG)=0$, then $\cD=\cG$ is the unique valid $\cI$-CCPG of $\cG$.
        \item If $\cI$ is a verifying intervention set of $\cG$, then $\cD=\cG$ is the unique valid $\cI$-CCPG of $\cG$.
    \end{itemize}
\end{restatable}



The key algorithmic contribution of our work, proven in Section~\ref{sec:5}, lies in an efficient algorithm that learns a valid CCPG with only \emph{polynomial} number of CI tests.

\begin{theorem}[Learning CCPG]\label{thm:main-1}
    Given observational data, there exists an efficient algorithm that performs at most $\cO(n^5)$ CI tests, and outputs a CCPG representation.
\end{theorem}

This result extends to the interventional setting as follows; the proof is given in Section~\ref{sec:6}.
\begin{theorem}[$\cI$-Learning CCPG]\label{thm:main-2}
    Given observational data and interventional data from interventions in $\cI$, there exists an efficient algorithm that performs at most $\cO(n^5)+|\cI|\cdot\cO(n^3)$ CI tests, and outputs an $\cI$-CCPG representation.
\end{theorem}

Combining these results with the properties of CCPG in Lemma~\ref{ccpg:prop}, this provides an efficient algorithm for learning the causal graph with polynomial number of conditional independence tests under certain cases, detailed below.

\begin{corollary}[\textbf{Causal Discovery with Polynomial CI Tests}]\label{coro}
For a DAG $\cG$ and its verifying intervention set $\cI$ (can be $\varnothing$), our algorithm recovers the full causal graph with at most $\cO(n^5)+|\cI|\cdot\cO(n^3)$ CI tests.
\end{corollary}

We remark here that \citet{eberhardt2012number} provides a construction of verifying intervention set of size $\log_2(n)$ that is independent of the underlying DAG $\cG$. Together with Corollary~\ref{coro}, this implies that our algorithm can learn the full causal graph with at most $\cO(n^5)$ CI tests.

To the best of our knowledge, our results present the first formal characterization of the information recoverable about general causal graphs using polynomial number of CI tests. 
Prior works showed that it is possible to learn \emph{sparse} causal graphs with $n^{\cO(k)}$ CI tests \cite{claassen2013learning, spirtes2000causation}. Here $k$ is an upper bound on the vertex degrees. Note that, the number of CI tests our algorithm requires, $\cO(n^5)$, is a polynomial of $n$ that is independent of any graph parameters. 


\subsection{Proxy V-structure and Meek Rule Statements}\label{sec:3-1}

In our derivations, we will make use of the following results, which we believe is of separate interest well beyond the scope of this work.

\begin{restatable}[Proxy V-Structure]{lemma}{proxyv}\label{lemma:v-generalization}
   Let $S\subseteq V$ and $u,v,z \in V \backslash S$. If $u \CI v|S$ and $u \not \CI v|S \cup \{z \}$, then $u,v \not \in \Des[z]$.\footnote{Similar argument is also proven in Lemma 1 of \citet{magliacane2016ancestral}.}
\end{restatable}

\begin{definition}[Prefix Vertex Set]
We call $S \subseteq V$ a prefix vertex set if it satisfies: for all $w \in S$, $\Anc[w] \cap \Bar{S} =\varnothing$ (vertices in $S$ appear first in the topological order).
\end{definition}

\begin{restatable}[Proxy Meek Rule 1]{lemma}{proxym}\label{lem:general-meekrule-1}
    Let $S$ be a prefix subset. If $u,v$ and $w$ are such that $u \in S$, $v, w \not \in S$ and $u \not \CI v |S$ and $u \CI w|S\ \cup \{ v\}$ then $v \not \in \Des[w]$.
\end{restatable}

The results stated above serve as proxy statements of v-structure and Meek Rule 1. Given the CI tests in the preceding lemmas, if we additionally have confirmed adjacencies between specific pairs of vertices, stronger statements could be made; e.g.,in Lemma~\ref{lemma:v-generalization}, we could conclude that $z$ is a child (or descendant) of both $u$ and $v$, and in Lemma~\ref{lem:general-meekrule-1} that $w$ is a child of $v$. However, since our lemma statements do not assume any knowledge of adjacency, we can only ascertain weaker statements, in the first case that $u$ and $v$ are not descendants of $z$, and in the second case that $v$ is not a descendant of $w$.

While our proxy results reveal weaker relationships among variables, they achieve this using a constant number of CI tests. Uncovering stronger relationships requires adjacency information, which may entail an exponential number of CI tests. Our main contribution lies in leveraging these weaker relationships, along with other favorable properties embedded in our algorithm, to implement the prefix vertex procedure and reconstruct the CCPG representation of the underlying causal graph using few CI tests.


%% file: content/prefix_subset_alg.tex
\section{Prefix Vertex Set}\label{sec:4}

The core component of our CCPG algorithm involves a procedure that produces a series of prefix vertex sets. 
We now show how such prefix vertex sets can be learned by performing few CI tests.

This will be helpful to obtain a CCPG representation of $\cG$, since the components $V_1,\dots,V_k$ satisfy that $V_1\cup\dots\cup V_i$ is a prefix vertex set for all $i\in [k]$.

\subsection{Algorithm for Learning}\label{sec:4-1}

We begin by presenting our algorithm for learning a prefix vertex set. This algorithm takes as input a prefix vertex set $S\subsetneq V$ (which can be $\varnothing$) and produces a larger prefix vertex set $S'$.

The analysis of Algorithm~\ref{alg:prefix_subset} will be provided in the next section. For an input prefix vertex set $S\subsetneq V$, it makes use of three types of CI tests, which we formalize below.\footnote{We note that $u,v,v'$ and $w$ in the definitions below are mutually distinct. We omit writing this for simplicity.}

\begin{definition}[Type-I Set $D_S$]
For all $w\in \bar{S}$, let $w\in D_S$ if and only if $u\CI v\mid S~\text{and}~u\not\CI v\mid S\cup\{w\}$ for some $u\in V$ and $v\in\bar{S}$.
\end{definition}

By the proxy v-structure in Lemma~\ref{lemma:v-generalization}, these two CI tests indicate that $v\not\in\Des[w]$. Therefore $w$ can potentially be a descendant of $v$. Thus $D_S$ contains vertices that are potential descendants of some other vertex in $\bar{S}$. We will rule out this set when searching for prefix $S'\supsetneq S$. Similarly, we can define a type-II set $E_S$.

\begin{definition}[Type-II Set $E_S$]
For all $w\in \bar{S}\setminus D_S$, let $w\in E_S$ if and only if $u\CI v'\mid S\cup\{v\}$ and $u\not\CI v'\mid S\cup\{v,w\}$ for some $u\in S$ and $v,v'\in\bar{S}\setminus D_S$.
\end{definition}

We also exclude any type-III set $F_S$, defined as follows.

\begin{definition}[Type-III Set $F_S$]
For all $w\in \bar{S}\setminus D_S$, let $w\in F_S$ if and only if $u\not\CI v\mid S$, $u\CI w\mid S\cup\{v\}$, and $v\not\CI w\mid S$ for some $u\in S$ and $v\in\bar{S}\setminus D_S$.
\end{definition}

By the proxy Meek Rule 1 in Lemma~\ref{lem:general-meekrule-1}, the first two CI tests $u\not\CI v\mid S$ and $u\CI w\mid S\cup\{v\}$ guarantee that $v\not\in\Des[w]$. The remaining CI test $v\not\CI w\mid S$ is to ensure that we do not exclude too many vertices in $F_S$, in particular $\src(\bar{S})$, as we show in the next section.

\begin{algorithm}[tbh!]
\begin{algorithmic}[1]
\caption{Learning a Prefix Vertex Set}
\label{alg:prefix_subset}
    \STATE \textbf{Input:} A prefix vertex set $S\subsetneq V$. CI queries from $\cG$. 
    \STATE \textbf{Output:} A prefix vertex set $S'$ such that $S'\supsetneq S$.
    \STATE Compute type-I set $D_S$.
    \STATE Compute type-II and III sets $E_S,F_S$.    
    \STATE Let $U = \bar{S}\setminus (D_S\cup E_S\cup F_S)$.
    \STATE \textbf{return} $S'=S\cup U$.
\end{algorithmic}
\end{algorithm}

\subsection{Correctness and Guarantees}\label{sec:4-2}

Algorithm~\ref{alg:prefix_subset} satisfies the following guarantees. All omitted proofs can be found in Appendix~\ref{app:b}.

\begin{theorem}\label{thm:prefix-subset}
Algorithm~\ref{alg:prefix_subset} outputs a prefix vertex set in $\cO(n^4)$ number of CI tests. In addition, this prefix vertex set contains all the remaining source nodes, i.e., $\src(\bar{S})\subset S'$.
\end{theorem}

For its proof, we will make use of the following properties of the type-I, II, and III sets (illustrated in Figure~\ref{fig:property-D-E-F}).

\begin{figure}[t]
    \centering
    \includegraphics[width=0.28\textwidth]{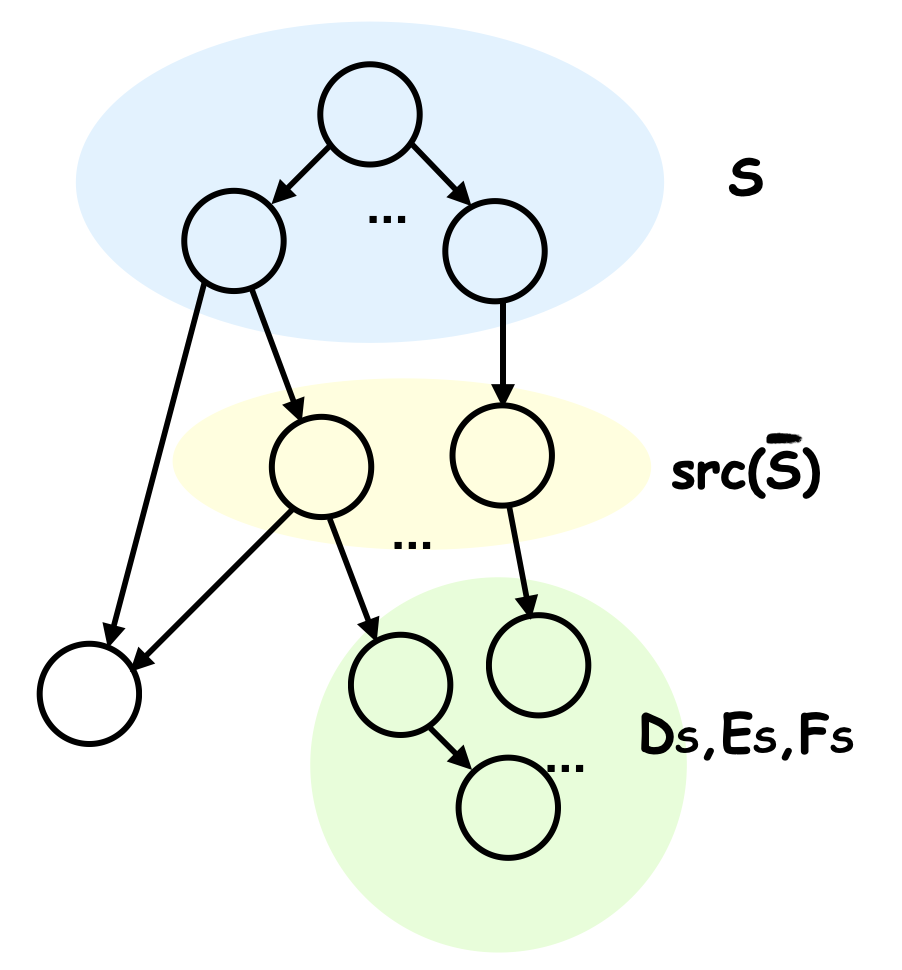}
    \caption{$D_S,E_S,F_S$ satisfy that (1) they contain all downstream vertices of any vertex in them; (2) they do not intersect with $\src(\bar{S})$.}\label{fig:property-D-E-F}
\end{figure}

\begin{restatable}{lemma}{lemDSESproperty}\label{lem:DS-ES-property}
    Let $S$ be a prefix vertex set. If $w\in D_S$, then $\Des[w]\subset D_S$. Furthermore, $D_S\cap \src(\bar{S})=\varnothing$. The same properties hold for $E_S$.
\end{restatable}

\begin{restatable}{lemma}{lemFSproperty}\label{lem:FS-property}
    Let $S$ be a prefix vertex set. If $w\in F_S$, then $\Des[w]\subset E_S\cup F_S$. Furthermore, $F_S\cap \src(\bar{S})=\varnothing$.
\end{restatable}

\begin{proof}[Proof of Theorem~\ref{thm:prefix-subset}]
    We first show that $S'$ returned by Algorithm~\ref{alg:prefix_subset} satisfies $\src(\bar{S})\subset S'$. By Lemmas~\ref{lem:DS-ES-property} and \ref{lem:FS-property}, we have $(D_S\cup E_S\cup F_S)\cap \src(\bar{S})=\varnothing$. As $\bar{S'}=D_S\cup E_S\cup F_S$, it must hold that $\src(\bar{S})\subset S'$.
    
    Next we show that $S'$ is a prefix vertex set. For this, we only need to show that $\forall w\in \bar{S'}$ and $y\in\Des(w)$, it holds that $y\in \bar{S'}$. Since $\bar{S'}=D_S\cup E_S\cup F_S$, one of the following three scenarios must hold: (1) if $w\in D_S$, then $y\in D_S$ by Lemma~\ref{lem:DS-ES-property}; (2) if $w\in E_S$, then $y\in E_S$ by Lemma~\ref{lem:DS-ES-property}; or (3) if $w\in F_S$, then $y\in E_S\cup F_S$ by Lemma~\ref{lem:FS-property}.

    Therefore $S'$ must be a prefix vertex set. We now bound the number of CI tests performed by Algorithm~\ref{alg:prefix_subset}: computing $D_{S}$ takes $\cO(n^3)$ CI tests; computing $E_S$ takes $\cO(n^4)$ CI tests; computing $F_S$ takes $\cO(n^3)$ CI tests, which completes the proof.
\end{proof}

\subsection{Relation to Covered Edges}\label{sec:4-3}

In Theorem~\ref{thm:prefix-subset}, we showed that $\src(\bar{S})\subset S'$. In fact, when there are no covered edges coming from $\src(\bar{S})$, one can show that $S'=\src(\bar{S})\cup S$ via the following \Cref{lem:main-lem}.

\begin{restatable}{lemma}{lemmainlem}\label{lem:main-lem}
    Let $S$ be a prefix vertex set. For $w\in\bar{S}\setminus D_S$, if $w\not\in\src(\bar{S})$ and there is \emph{no} covered edge from $\Anc[w]\cap\src(\bar{S})$ to $\Anc[w]$, then $w\in E_S\cup F_S$.
\end{restatable}

\begin{corollary}\label{cor:sec4-3}
    Let $S$ be a prefix vertex set. If there is \emph{no} covered edge in $\bar{S}$, then $S'=\src(\bar{S})\cup S$.
\end{corollary}

This result will be useful when deriving CCPG representations using Algorithm~\ref{alg:prefix_subset}. To see this, consider the simple case where there is no covered edge in $\cG$. Then running Algorithm~\ref{alg:prefix_subset} with $S=\varnothing$ we can learn $\src(V)$ by Corollary~\ref{cor:sec4-3}. Then running Algorithm~\ref{alg:prefix_subset} with $S=\src(V)$, we can learn the source vertices of $V\setminus\src(V)$. Applying this iteratively, we can obtain the ground-truth topological order of $\cG$. As a consequence, one can easily learn $\cG$,\footnote{For $i<j$ in the topological order, the corresponding edge $v_i\rightarrow v_j\in\cG$ iff $v_i\not\CI v_j\mid v_1,\dots,v_{j-1}$.} which is the sole CCPG representation as there is no covered edge.

%% file: content/TCG_alg.tex
\section{Learning Causally Consistent Partition Graph Representations}\label{sec:5}

We now present our algorithm for learning causally consistent partition graph representations. All omitted proofs can be found in Appendix~\ref{app:c}.

Algorithm~\ref{alg:prefix_cluster} contains three parts, indicated by colored boxes below. In the first part, we use Algorithm~\ref{alg:prefix_subset} iteratively to learn prefix vertex sets of the form $\varnothing\subsetneq\dots S\subsetneq S'\subsetneq\dots\subsetneq V$. In the second part, we break each $S'\setminus S$ into smaller components that are pairwise independent given $S$. As we will see, these components satisfy the property of CCPG in Definition~\ref{def:TCG}. In the third part, we build the acyclic graph between the CCPG components.

\begin{algorithm}
\begin{algorithmic}[1]
\caption{Learning a CCPG Representation}
\label{alg:prefix_cluster}
    \STATE \textbf{Input:} CI queries from $\cG$.
    \STATE \textbf{Output:} A CCPG representation of $\cG$.
    \tikzmk{A}
    \STATE Set $S=\varnothing$ and $\cS$ as empty ordered list.
    \WHILE{$S\neq V$}
    \STATE Run Algorithm~\ref{alg:prefix_subset} on $S$ to obtain $S'$.
    \STATE Add $S'\setminus S$ to the end of $\cS$.
    \STATE Update $S=S'$. 
    \ENDWHILE
    \tikzmk{B}\boxit{mypink}
    \tikzmk{A}
    \STATE Initialize $l_1=1$.
    \FOR{$S_i$ in $\cS=[S_1,\dots,S_m]$}
    \STATE Create empty graph $\cT$ on vertices in $S_i$.
    \STATE Add $v-w$ to $\cT$ iff $v\not\CI w\mid S_1\cup\dots\cup S_{i-1}$.
    \STATE Split $S_i$ into components $V_{l_i},V_{l_i+1},\dots,V_{l_{i+1}}$ based on connected components in $\cT$.
    \ENDFOR
    \tikzmk{B}\boxitt{myyellow}
    \tikzmk{A}
    \STATE Create empty graph $\cD$ on vertices in $[l_{m+1}]$.
    \FOR{$V_i,V_j$ with $i<j$}
    \STATE Add $i\to j$ to $\cD$ iff $V_i\not\CI V_j\mid V_{1}\cup\dots\cup V_{j-1}$.
    \ENDFOR
    \tikzmk{B}\boxittt{myblue}
    \STATE \textbf{return} $V_1,\dots,V_{l_{m+1}}$ and $\cD$
\end{algorithmic}
\end{algorithm}

To show that the components we obtain in the second part satisfy the property of CCPG, we will use the following result. Essentially, we can show that the subgraph $\cG[S'\setminus S]$ can be split into connected subgraphs based on individual vertices in $\src(S'\setminus S)$ ($=\src(\bar{S})$ by Theorem~\ref{thm:prefix-subset}). This is illustrated in Figure \ref{fig:S'-S}.

\begin{restatable}{lemma}{lemonenodeancsrc}\label{lem:one-node-anc-src}
    Let $S$ be a prefix subset. For any $w \in \bar{S} \backslash D_{S}$, there is $|\Anc[w] \cap \src(\bar{S})| =1~.$ Furthermore, for any other $w'\in\bar{S}\setminus D_S$, there is $w\not\CI w'\mid S$ \emph{if and only if} $\Anc[w] \cap \src(\bar{S})= \Anc[w'] \cap \src(\bar{S})$.
\end{restatable}

\begin{figure}[t]
    \centering
    \includegraphics[width=0.42\textwidth]{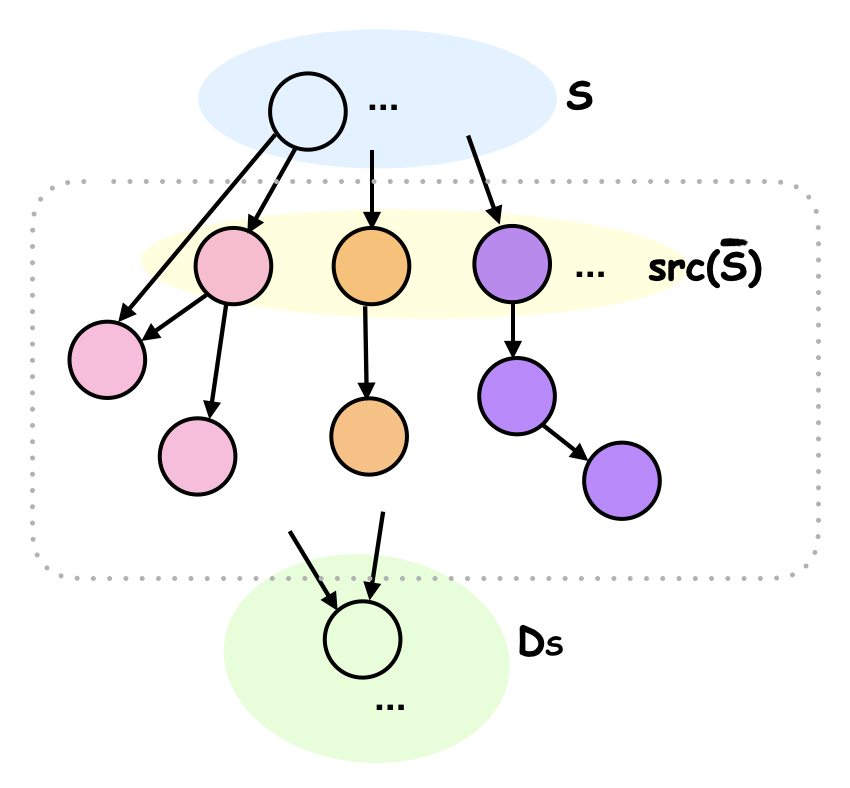}
    \caption{Illustration of $\bar{S}\setminus D_S$ (inside the dashed box). It can be split into connected subgraphs based on vertices in $\src(\bar{S})$ (indicated by the fill color of each vertex in $\bar{S}\setminus D_S$).}\label{fig:S'-S}
\end{figure}

Based on this result, we can establish that Algorithm~\ref{alg:prefix_cluster} outputs a CCPG representation by proving Theorem~\ref{thm:main-1}.

\begin{proof}[Proof of Theorem~\ref{thm:main-1}]
We show that Algorithm~\ref{alg:prefix_cluster} outputs a CCPG representation of $\cG$. Since it runs in polynomial time and uses $\cO(n^5)$ CI tests, this will prove Theorem~\ref{thm:main-1}.

\textbf{Intra-component Property.} Consider $S_i$. Denote $S_1\cup\dots\cup S_{i-1}=S$. We will show that $S_i$ is split into $\Des[s]\cap S_i$ for each $s\in\src(\bar{S})$. Since $S_i\subset \bar{S}$, we have $S_i=\cup_{s\in\src(\bar{S})} (\Des[s]\cap S_i)$. By Theorem~\ref{thm:prefix-subset}, we know that $S$ is a prefix vertex set. Since $S_i\subset \bar{S}\setminus D_S$, by Lemma~\ref{lem:one-node-anc-src}, $\Des[s]\cap S_i$ for different $s\in\src(\bar{S})$ are disjoint. Furthermore, by Theorem~\ref{thm:prefix-subset}, these disjoint sets are non-empty because $s\in \Des[s]\cap S_i$. Therefore $\src(\Des[s]\cap S_i)=\{s\}$ and we have $|\src(\Des[s]\cap S_i)|=1$. Thus by Lemma~\ref{lem:one-node-anc-src}, each of $V_{l_i},V_{l_i+1},\dots,V_{l_{i+1}}$ corresponds to $\Des[s]\cap S_i$ for some $s\in\src(\bar{S})$.

If $|\Des[s]\cap S_i|>1$, then it contains a vertex that is not $s$. Suppose $w\in\Des(s)\cap S_i$. By Lemmas~\ref{lem:main-lem} and~\ref{lem:one-node-anc-src}, there is a covered edge from $s$ to some vertex $x\in\Anc[w]$. Since $S$ and $S\cup S_i$ are prefix vertex sets, $s,w\in S_i$ implies $x\in S_i$. Thus $x\in \Des[s]\cap S_i$ and there is a covered edge $s\to x$ in $\Des[s]\cap S_i$. Thus, we have proven that each component in $V_1,\dots,V_{l_{m+1}}$ satisfies the first property of CCPG.

\textbf{Inter-component Property.} For each $V_i$, denote the subset that it came from as $S_{h_i}$, i.e., $V_i\subset S_{h_i}$. By the construction of $\cD$ in Algorithm~\ref{alg:prefix_cluster}, no edge $i\to j$ in $\cD$ means either $i>j$ or $V_i\CI V_j\mid V_1\cup\dots\cup V_{j-1}$. If $i>j$ and $h_i\neq h_j$, then by the fact that $S_1\cup\dots\cup S_{h_i}$ is a prefix vertex set, there is no edge from $V_i$ to $V_j$ in $\cG$. If $h_i=h_j=h$, then we know that there exists $s_i\neq s_j$ such that $V_i=\Des[s_i]\cap S_h$ and $V_j=\Des[s_j]\cap S_h$. If there is an edge from $V_i$ to $V_j$ in $\cG$, denoted as $v\to w$, then $w\in\Des[v]\subset\Des[s_i]$ which means $V_i\cap V_j\neq \varnothing$, a contradiction. Therefore there is no edge from $V_i$ to $V_j$ in $\cG$. If $V_i\CI V_j\mid V_1\cup\dots\cup V_{j-1}$, then clearly there is no edge from $V_i$ to $V_j$ in $\cG$. Thus when there is no $i\to j$ in $\cD$, there is no edge from $V_i$ to $V_j$ in $\cG$.

If there is an edge $i\to j$ in $\cD$, then we have $i<j$ and $V_i\not\CI V_j\mid V_1\cup\dots\cup V_{j-1}$. Note that since $S_1\cup\dots\cup S_{h_j}$ is prefixed, it holds that $\Pa(V_{j})\subset S_{1}\cup\dots \cup S_{h_j}$ and $\Des[V_j]\cap (S_1\cup\dots S_{h_j-1})=\varnothing$. As shown above, there is no edge between any other $V_{j'}\subset S_{h_j}$ and $V_j$. Thus we have $\Pa(V_{j})\subset V_1\cup\dots \cup V_j$ and $\Des[V_j]\cap(V_1\cup\dots\cup V_{j-1})=\varnothing$. Similarly $\Pa(V_i)\subset  V_1\cup\dots \cup V_i$. Thus by the local Markov property and Bayes rule,\footnote{See Lemma~\ref{lem:general-local-markov} in Appendix~\ref{app:a}.} $V_i\not\CI V_j\mid V_1\cup\dots\cup V_{j-1}$ only if there is a direct edge $u\to v$ from $u\in V_i$ to $v\in V_j$. We now show that there is also an edge $u\to s$ where $\{s\}=\src(V_j)$. Assume on the contrary that there is no edge $u\to s$. Since $s$ is the source node of $V_j$ and thus a source node of $S_{h_j}$, we have from $\Pa(V_{j})\subset S_1\cup\dots \cup S_{h_j}$ that $\Pa(s)\subset S_1\cup\dots\cup S_{h_j-1}$. In addition, $\Des[s]\cap (S_1\cup\dots\cup S_{h_j-1})=\varnothing$. Thus by the local Markov property, $u\CI s\mid S_1\cup\dots\cup S_{h_j-1}$. However as $u\to v$ and $v\in\Des[s]$, we have $u\not\CI s\mid S_1\cup\dots\cup S_{h_j-1}\cup\{v\}$. As a consequence, $v\in D_{S_1\cup\dots\cup S_{h_j-1}}$. By Algorithm~\ref{alg:prefix_subset}, it is impossible that $v\in S_{h_j}$, a contradiction. Therefore we must have $u\in\Pa(s)$. This proves that $\cD$ satisfies the second property of CCPG.
\end{proof}

%% file: content/interventions.tex
\section{Interventions}\label{sec:6}

In Sections~\ref{sec:4} and \ref{sec:5}, we showed how to learn a CCPG representation using observational data. Here we generalize these methods to interventions. This results in a more refined $\cI$-CCPG representation of $\cG$.  All omitted proofs can be found in Appendix~\ref{app:d}.

\subsection{Learning Refined Prefix Vertex Sets}\label{sec:6-1}

In Section~\ref{sec:4-1}, we obtained a prefix vertex set by excluding three types of sets $D_S,E_S$ and $F_S$. With interventions, we can exclude an additional set. To define it, we first characterize what can be learned using interventions.

\begin{restatable}{lemma}{lemsixone}\label{lem:6-1}
    Let $S$ be a prefixed vertex set. Given an intervention on vertices $I\subseteq V$. For any $u\not\in I$, we have $u\not\CI_I v$ for some $v\in I\setminus S$ if and only if $u\in\Des(I\setminus S)\setminus (I\setminus S)$.
\end{restatable}

Note that when the number of CI tests is not restricted, one can learn all edges cut by $I$ as well as their orientations. Lemma~\ref{lem:6-1} shows that it is possible to learn the joint set of descendants using $\cO(n^2)$ CI tests.

With this set and a few more CI tests, it is possible to learn additional directional information.

\begin{restatable}{lemma}{leminthelper}\label{lem:int-helper}
    Let $S$ be a prefix subset. Given an intervention $I$ and $v\in I\setminus S$, denote $H_{S}^I(v)=\{u\not\in S\cup \Des[I\setminus S]:~ u\not\CI v\mid V\setminus\Des[I\setminus S]\}$.\footnote{Note that $\Des[I\setminus S]$ can be obtained via Lemma~\ref{lem:6-1} and $(\Des(I\setminus S)\setminus (I\setminus S))\cup (I\setminus S)$.} Then $Pa(v)\setminus (S\cup \Des[I\setminus S])\subseteq H_S^I(v) \subseteq \Anc(v)\setminus (S\cup \Des[I\setminus S])$.
\end{restatable}

Such sets can be useful when deciding if a target of $I\setminus S$ should be excluded when learning a prefix vertex subset $S'\supset S$. Formally, we define the type-IV set as follows.

\begin{definition}[Type-IV Set $J_S^I$]
    Let $S$ be a prefix vertex set. For an intervention $I$, let $J_S^I=\Des(I\setminus S)\cup \{v\in I\setminus S:~ H_{S}^{I}(v)\cap\bar{S}\neq \varnothing\}$.
\end{definition}

Analogous to Lemma~\ref{lem:DS-ES-property}, we can show that $J_S^I$ satisfies similar properties (illustrated in Figure~\ref{fig:property-J}).

\begin{figure}[t]
    \centering
    \includegraphics[width=0.43\textwidth]{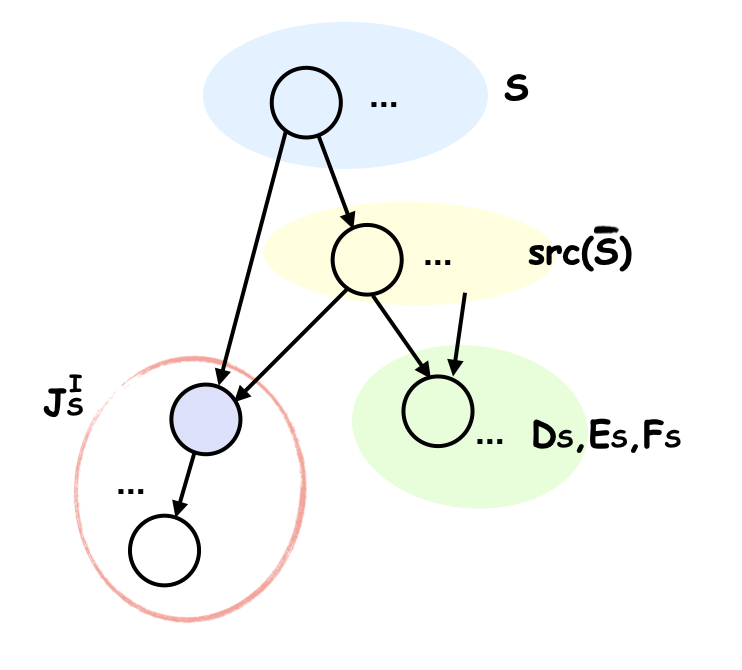}
    \caption{Illustration of $J_S^I$, where $I$ is indicated by the purple circle. $J_S^I$ satisfies similar properties as $D_S,E_S$ and $F_S$.}\label{fig:property-J}
\end{figure}

\begin{restatable}{lemma}{lemJSI}\label{lem:J-S-I}
Let $S$ be a prefix vertex set. Then $\forall w\in J_{S}^{I}$, we have $\Des(w)\subseteq J_{S}^{I}$. Furthermore, $J_{S}^{I}\cap\src(\bar{S})=\varnothing$.
\end{restatable}

Therefore we can exclude $J_S^I$ as well, which results in the following modification of Algorithm~\ref{alg:prefix_subset} and for which we can show similar guarantees using Lemma~\ref{lem:J-S-I}.

\begin{algorithm}[tbh!]
\begin{algorithmic}[1]
\caption{Learning a Prefix Vertex Set (w. Interventions)}
\label{alg:prefix_subset_int}
    \STATE \textbf{Input:} A prefix vertex set $S \subsetneq V$. CI queries from $\cG$ and $\cG^I$ for each intervention $I\in\cI$.
    \STATE \textbf{Output:} A prefix vertex set $S'$ such that $S'\supsetneq S$.
    \FOR{$I\in\cI$}
    \STATE Compute type-IV set $J_{S}^{I}$.
    \ENDFOR
    \STATE Compute type-I set $D_S$.
    \STATE Compute type-II and III sets $E_S,F_S$.
    \STATE Let $U' = \bar{S}\setminus \big(\cup_{I\in\cI} J_{S}^{I}\cup D_S\cup E_S\cup F_S\big)$.
    \STATE \textbf{return} $S'=S\cup U'$.
\end{algorithmic}
\end{algorithm}

\begin{restatable}{theorem}{thmprefixsubsetint}\label{thm:prefix-subset-int}
Algorithm~\ref{alg:prefix_subset_int} outputs a prefix vertex set in $\cO(n^4)+|\cI|\cO(n^2)$ CI tests. In addition, this prefix vertex set contains all the remaining source nodes, i.e., $\src(\bar{S})\subset S'$.
\end{restatable}

\begin{proof}
Similar to the proof of Theorem~\ref{thm:main-1}, we can use the additional Lemma~\ref{lem:J-S-I} to show that: (1) $S'$ is a prefix vertex set, (2) it contains $\src(\bar{S})$. 
We now bound the number of CI tests performed by Algorithm~\ref{alg:prefix_subset_int}: note that it bears the additional need to compute $J_S^I$ compared to Algorithm~\ref{alg:prefix_subset}. By Lemmas~\ref{lem:6-1} and \ref{lem:int-helper}, computing $J_S^I$ for each $I\in\cI$ takes $\cO(n^2)$. Thus the total complexity is $\cO(n^4)+|\cI|\cO(n^2)$.
\end{proof}

In addition, we can show the following property for covered edges coming from $\src(\bar{S})$.

\begin{restatable}{lemma}{lemintcoveredsep}\label{lem:int-covered-sep}
    Let $S$ be a prefix vertex set, $v\in\src(\bar{S})$ and $v\to w$ be a covered edge. Given any intervention $I$, if $|I\cap\{v,w\}|=1$, then $w\in J_{S}^{I}$.
\end{restatable}

\begin{proof}
    If $I\cap\{v,w\}=\{v\}$, then $w\in\Des(I\setminus S)$. If $I\cap\{v,w\}=\{w\}$, then $w\in I\setminus S$. We will show that $ H_{S}^{I}(w)\cap \bar{S}\neq \varnothing$. This in turn proves that $w\in J_S^I$. Note that since $v\in\src(\bar{S})$, we have $v\not\in S\cup\Des[I\setminus S]$. Furthermore, since $v\to w$ is an edge, we have $v\in\Pa(w)$. Thus by Lemma~\ref{lem:int-helper}, we have $v\in\Pa(w)\setminus(S\cup\Des[I\setminus S])\subseteq H_{S}^{I}(w)$.
\end{proof}

\subsection{Learning $\cI$-CCPG Representations}\label{sec:6-2}

We obtain the final algorithm by plugging Algorithm~\ref{alg:prefix_subset_int} into Algorithm~\ref{alg:prefix_cluster}, i.e., replacing Algorithm~\ref{alg:prefix_subset} in line 5 by Algorithm~\ref{alg:prefix_subset_int}. 

\begin{proof}[Proof of Theorem~\ref{thm:main-2}]
By Theorem~\ref{thm:prefix-subset-int}, for any $i$, set $S=S_1\cup\dots\cup S_{i-1}$ is a prefix vertex set. Therefore using the proof in Theorem~\ref{thm:main-1} and the fact that $\src(\bar{S})\subset S\cup S_i$ (Theorem~\ref{thm:prefix-subset-int}), we immediately have $|\src(V_i)|=1$ and $\cD$ satisfies the second property of $\cI$-CCPG in Definition~\ref{def:TCG}. 

We now show that when $|V_i|>1$, subgraph $\cG[V_i]$ has at least one unintervened covered edge. Suppose on the contrary that $\cG[V_i]$ has no unintervened covered edge. Denote $\{s\}=\src(V_i)$. Let $w\in V_i$ such that $w\neq s$. Similar to Theorem~\ref{thm:main-1}, we can show that $w\in\Des(s)$. By Lemma~\ref{lem:main-lem}, there is a covered edge from $s$ to some vertex $x\in\Anc[w]$ and $x\in V_i$. Since $\cG[V_i]$ has no unintervened covered edge, $s\to x$ is intervened by some $I\in\cI$. Then by Lemma~\ref{lem:int-covered-sep}, we have $x\in J^{I}_S$. This contradicts $x\in V_i$.
\end{proof}

%% file: content/experiment.tex
\section{Experiments}\label{sec:exp}

In this section, we test our proposed method and compare it to existing causal discovery methods on synthetic data and a toy real-world example. Source code for these results can be found at \url{https://github.com/uhlerlab/CCPG}.

\subsection{Synthetic Data}

In these experiments, we consider the observational setting and generate samples from linear causal models with additive Gaussian noise, governed by identifiable causal graphs, in particular, in-star-shaped DAGs with varying number of nodes. In such settings, with enough samples, our algorithm is guaranteed to return the ground-truth DAG (\Cref{coro}).

We compare our \Cref{alg:prefix_cluster}, termed \texttt{CCPG}, to other constraint-based and hybrid methods that rely on conditional independence tests. The constraint-based methods we include are: \texttt{PC} \cite{spirtes2000causation}, \texttt{FCI} \cite{spirtes1999algorithm}, and \texttt{RFCI} \cite{colombo2012learning}, where we use the order-independent variants from \citet{colombo2014order} (i.e., ``\texttt{stable}'' versions). The hybrid methods we included are: \texttt{GSP} with depth $=4$ and depth $=\infty$ \cite{solus2021consistency}. Implementation details can be found in \Cref{app:e}.

\begin{figure}[h]
    \centering
    \includegraphics[width=0.38\textwidth]{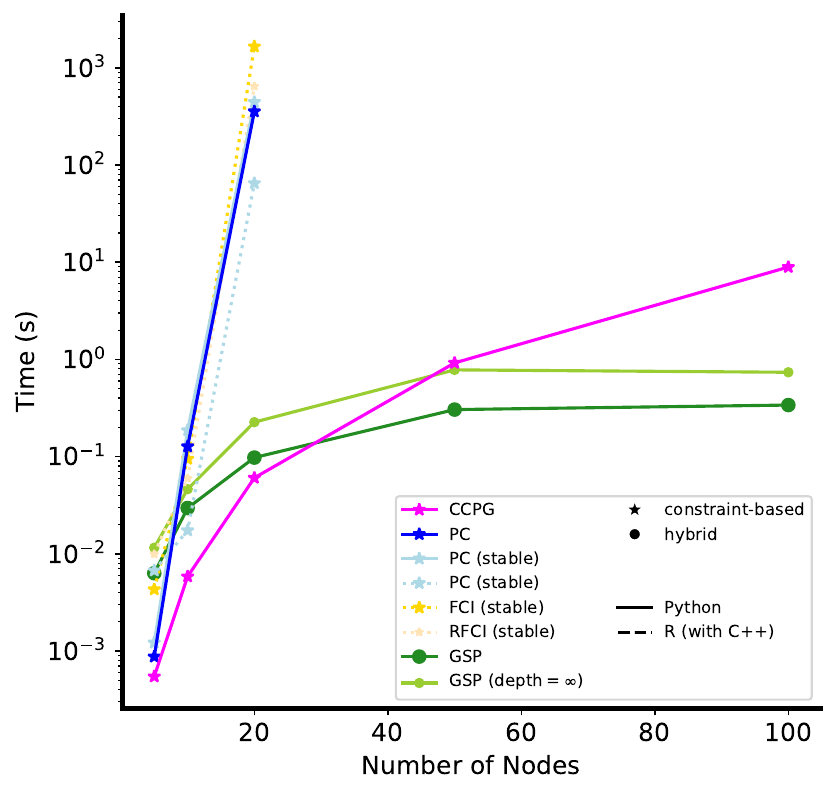}
    \caption{\textbf{Runtime comparison across graphs of different sizes.} \texttt{CCPG} is as fast as hybrid methods (\texttt{GSP}) and significantly more efficient than constraint-based methods (\texttt{PC}, \texttt{FCI}, and \texttt{RFCI}). The programming language behind each implementation is indicated by dashed or solid lines (see \Cref{app:e} for details).}\label{fig:runtime}
\end{figure}

\textbf{Runtime Analysis.} To test the computational efficiency of our method, we generated 100k samples across graphs of different sizes and reported the runtime averaged across five runs. \Cref{fig:runtime} shows that \texttt{CCPG} is as fast as hybrid methods (\texttt{GSP}) and significantly more efficient than constraint-based methods (\texttt{PC}, \texttt{FCI}, and \texttt{RFCI}), which can only scale to 20 nodes with reasonable runtime.

\textbf{Sample Complexity Analysis.} To test the sample efficiency of \texttt{CCPG}, we consider a 10-node in-star-shaped DAG. For each method, we increase the number of samples until it returns the ground-truth DAG and repeat this procedure for five runs. \Cref{fig:nsample} shows that \texttt{CCPG} requires the least number of samples to recover the true graph. 

In general, constraint-based methods (\texttt{PC}, \texttt{FCI}, and \texttt{RFCI}) have better sample efficiency than hybrid methods, while being significantly more runtime inefficient. In comparison, \texttt{CCPG}, with its polynomial-number of CI test guarantee, enjoys low sample complexity and low runtime complexity.

\begin{figure}[t]
    \centering
    \includegraphics[width=0.35\textwidth]{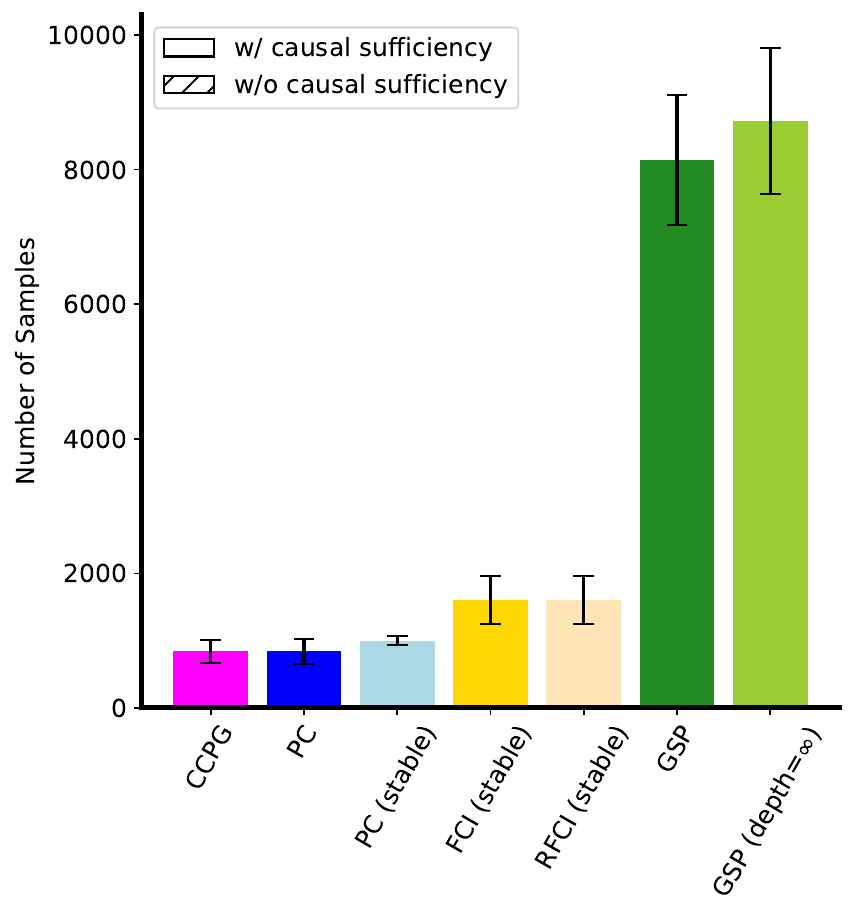}
    \caption{\textbf{Sample complexity comparison.} \texttt{CCPG} requires the least number of samples to recover the true graph. Methods that do not assume causal sufficiency are indicated by shading (see \Cref{app:e} for details).}\label{fig:nsample}
\end{figure}

\subsection{A Real-world Example}

To illustrate the utility of the coarser representation learned by \texttt{CCPG} in real-world settings, we include a simple 6-variable Airfoil example \cite{asuncion2007uci,lam2022greedy}; see \citet{lam2022greedy} for a detailed description of this example. Although there is no known ground-truth DAG in this setting, a few causal relations are known: (1) velocity, chord, and attack should be source nodes; (2) pressure is downstream of all other nodes.

The coarser representation learned by \texttt{CCPG} is shown in \Cref{fig:ccpg-exp}. Compared to \texttt{PC}, which returns the partially directed graph in \Cref{fig:pc}, \texttt{CCPG} seems to be more consistent with the known causal relations while it contains less information due to a coarser representation.

\begin{figure}[t]
    \centering
    \includegraphics[width=0.25\textwidth]{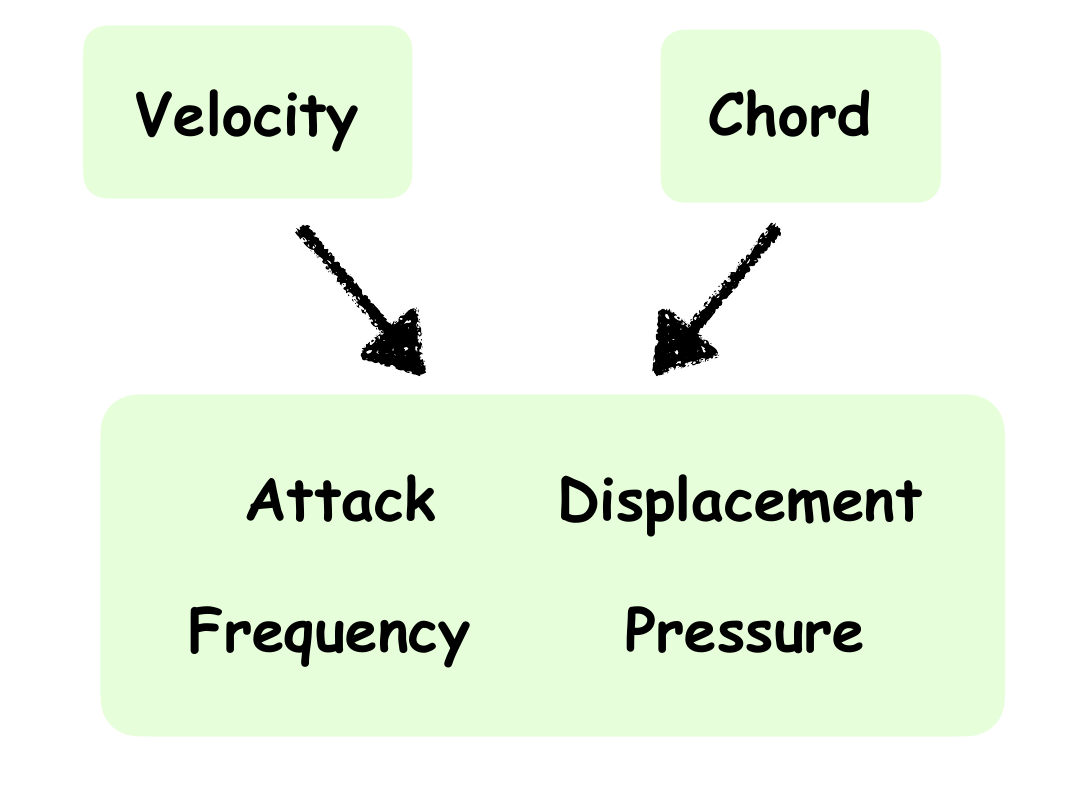}
    \caption{Learned coarser representation by \texttt{CCPG} in the Airfoil example.}\label{fig:ccpg-exp}
\end{figure}

\begin{figure}[t]
    \centering
    \includegraphics[width=0.3\textwidth]{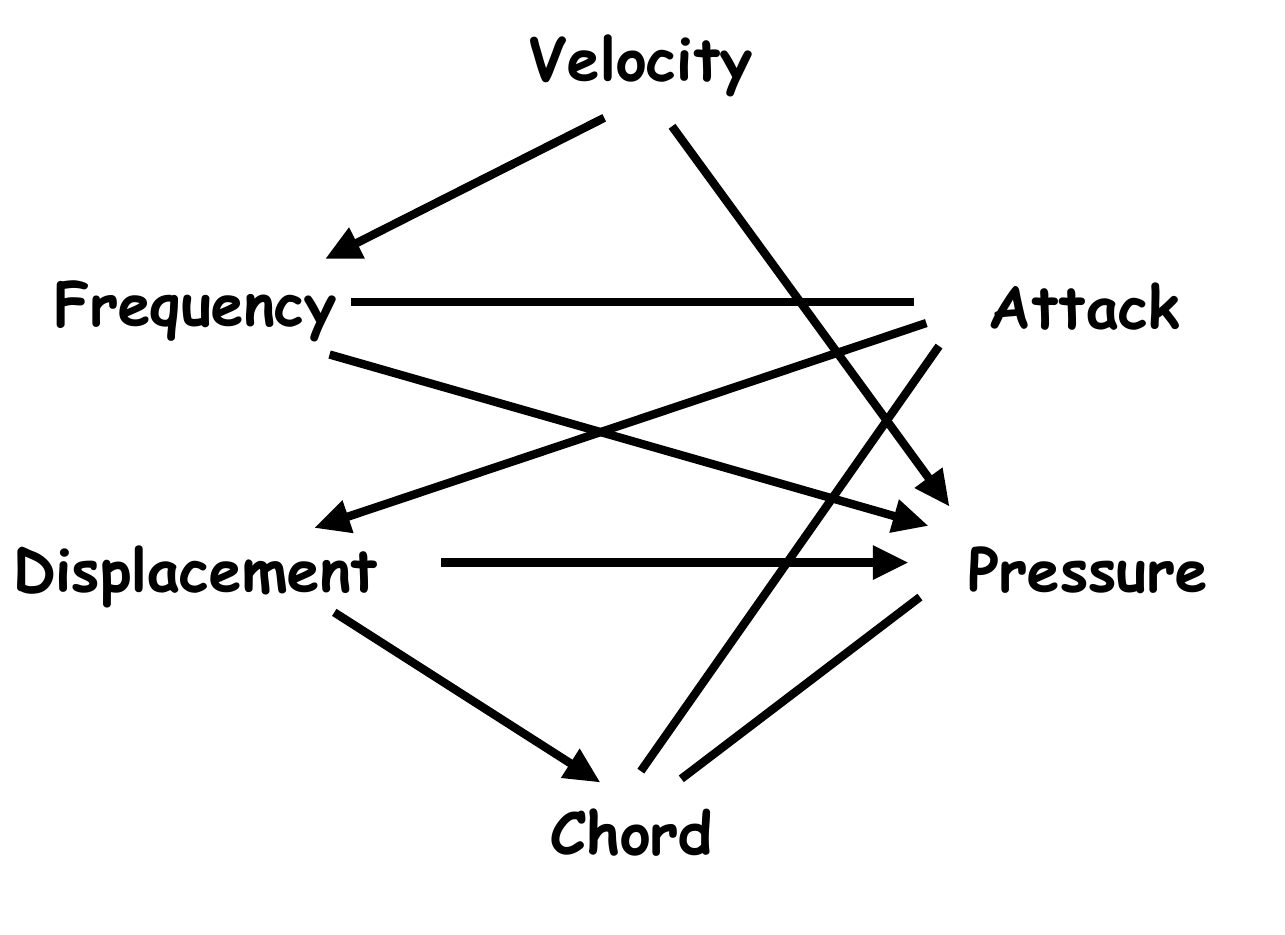}
    \caption{Learned partially directed graph by \texttt{PC} in the Airfoil example.}\label{fig:pc}
\end{figure}

%% file: content/discussion.tex
\section{Discussion}\label{sec:7}
In our work, we studied causal structure learning under the constraint of fewer CI tests. Since exact structure learning may demand an exponential number of CI tests, we defined a representation (CCPG) that captures partial but crucial information about the underlying causal graph. Moreover, we provided an efficient algorithm that recovers a CCPG representation in a polynomial number of CI tests. This result enabled us to design efficient algorithms for the full recovery of causal graphs in two specific settings, utilizing only a polynomial number of CI tests.

We hope that our work will motivate further exploration of the causal discovery problem under the constraint of fewer CI tests, extending to various settings, including those involving latent variables. Furthermore, our research establishes a foundation for addressing the search problem\footnote{The search problem involves finding the minimum set of interventions that orient the entire causal graph.} with reduced tests, suggesting the potential existence of search algorithms capable of recovering the causal graph with a polynomial number of independence tests while performing an approximately optimal number of interventions.

%% file: appendix/supplement.tex
\section{Useful Lemmas}\label{app:a}

\subsection{Proof of Lemma~\ref{ccpg:prop}}
\ccpgprop*
\begin{proof}

Note that if $|V_i|=1$, then CCPG is actually $\cG$. If $|V_i|>1$, then it means that there is a covered edge in this subset that is not intervened. By Proposition~\ref{prop:1}, this is impossible when $\cI$ is a verifying intervention set.
\end{proof}

\subsection{Proof of Lemma~\ref{lemma:v-generalization}}

\proxyv*
\begin{proof}
    Let $P$ be an active path (that carries dependency) between $u$ and $v$ when conditioned on $S \cup \{z \}$. Since $u \ci v|S$ and $u \not \CI v|S \cup \{z \}$, we get that there exist a set of vertices $w_1\dots w_{k}$ that are colliders on $P$ and satisfy: $\Des[w_i] \cap S=\varnothing$ and $z \in \Des[w_i]$. 
    
    Consider the path $P$ and note that it takes the form $P=u-\dots\to w_1 \gets \dots  \to w_i \gets \dots \to w_{k} \gets \dots - v$. 
    Define $P_1= u-\dots \to w_1$ and $P_2=w_{k} \gets \dots - v$. Since all the colliders on paths $P_1, P_2$ are in or have their descendants in $S$ and all non-colliders do not belong to $S$, we have that $P_1,P_2$ are active given $S$.
    We prove our lemma using the proof by contradiction strategy. For contradiction, let us assume that one of the vertices in $\{ u,v\}$ belong to the set $\Des[z]$ and without loss of generality let that vertex be $u$. Then since $z\in\Des[w_i]$ for all $i$, there must be $u\in\Des[w_i]$ for all $i$.

    Since $u \in \Des[w_{k}]$, let $Q$ be the directed path in the graph that connects $w_k$ to $u$. Note that all the vertices in path $Q$ are all descendants of $w_k$ and they do not belong to the set $S$ (because $\Des[w_k]\cap S=\varnothing$). Now consider the path $(P_2,Q)$ that connects vertices $v$ and $u$ and note that the vertex $w$ is a non-collider on the new path $(P_2,Q)$. Since $P_2$ and $Q$ are active paths given $S$ and since $w \not \in S$, we immediately get that the path $(P_2,Q)$ is an active path given $S$, which further implies that $u \not \ci v|S$; a contradiction and we conclude the proof.
\end{proof}

\subsection{Proof of Lemma~\ref{lem:general-meekrule-1}}

\proxym*
\begin{proof}
    Suppose on the contrary that there is $u,v,w$ such that  $u \in S$, $v, w \not \in S$ and $u \not \CI v |S$ and $u \CI w|S\ \cup \{ v\}$ and $v\in\Des[w]$.

    Since $u\not\CI v\mid S$, let $P$ be the active path connecting $u$ and $v$ given $S$. As $u\in S$ and $v\not\in S$, there is an edge $u'- v'$ on $P$ such that $u'\in S$ but $v'\not\in S$. 
    
    Now denote the vertex on $P$ that is immediate next to $v$ as $x$. If $x\leftarrow v$, then there must be a collider on $P$ between $u'$ and $v$ as $S$ is a prefix subset where $u'\in S$ and $v\not\in S$. Let $y$ be the collider on $P$ between $u'$ and $v$ that is closest to $u'$, then $u'\to v'\dots\to y$. Since $P$ is active given $S$, the collider $y\in \Anc[S]$. However $v'\in \Anc[y]\subseteq\Anc[S]$ and $v'\notin S$, a contradiction to $S$ being prefix. Thus there must be $x\to v$ on $P$.

    Since we assumed on the contrary that $v\in\Des[w]$, we can consider the path $Q$ joined by $P$ and the directed path from $w$ to $v$. Compared to $P$, the path $P$ has one additional collider $v$, and has a few additional colliders that lie between $v$ and $w$ which are not in $S$ (as they are all descendants of $w$ and $w\not\in S$). Therefore $Q$ is active given $S\cup\{v\}$. This means $u\not\CI w \mid S\cup \{v\}$, a contradiction. 
\end{proof}

\subsection{Additional Lemma}
In addition, we will make use of the following lemma.

\begin{lemma}\label{lem:general-local-markov}
    For disjoint sets $A,B,C\subset V$, if $\Pa(A)\subset C\cup A$, $\Pa(B)\subset C\cup B$ and $\Des(B)\cap C=\varnothing$, then $A\CI B\mid C$.
\end{lemma}

\begin{proof}
    Assume without loss of generality that the vertices in $B$ have the following topological order $b_1,\dots,b_m$.
    Then for any $i\in[m]$, by the local Markov property \cite{spirtes1989causality}, we have $A\CI b_i\mid C\cup \{b_1,\dots,b_{i-1}\}$.
    Therefore using Bayes rule, we have
    \begin{align*}
        P(B\mid C,A) & = P(b_1\mid C,A)P(b_2\mid C,A,b_1)\dots P(b_m\mid C,A,b_1,\dots,b_{m-1}) \\
        & = P(b_1\mid C,A)P(b_2\mid C,b_1)\dots P(b_m\mid C,b_1,\dots,b_{m-1}) \\
        & = P(B\mid C),
    \end{align*}
    and thus $A\CI B\mid C$, which completes the proof.
\end{proof}

\section{Missing Proofs of Prefix Vertex Set}\label{app:b}

\subsection{Proof of Lemma~\ref{lem:DS-ES-property}}

We restate the lemma below.
\lemDSESproperty*

\subsubsection{Type-I Set $D_S$}
\begin{proof}[Proof of Lemma~\ref{lem:DS-ES-property} for $D_S$]
We first show that if $w\in D_S$, then it must hold that $\Des[w]\subset D_S$: since $w\in D_{S}$, there exists a vertex $u\in V$ and $v\in \bar{S}$ such that $u \CI v\mid S$ and $u \not\CI v\mid S\cup\{w\}$. We now show that for any $x\in \Des[w]$, we have $u \not\CI v\mid S\cup\{x\}$.

Since $u \not\CI v\mid S\cup\{w\}$, there is a path $P$ from $u$ to $v$ that is active given $S\cup\{w\}$. Therefore, all non-colliders on $P$ are not in $S\cup\{w\}$ and all colliders on $P$ are in $\Anc[S\cup\{w\}]$. Since $x\in \Des[w]$, all colliders on $P$ are in $\Anc[S\cup\{x\}]$. If all non-colliders are not in $S\cup\{x\}$, then $P$ is an active path from $u$ to $v$ given $S\cup\{x\}$, and thus $u \not\CI v\mid S\cup\{x\}$. Otherwise there is a non-collider on $P$ that is $x$.

Since $u\CI v\mid S$, the path $P$ is inactive given $S$. From above we know that all non-colliders on $P$ are not in $S$. Therefore there exists a collider on $P$ that is not in $\Anc[S]$. Suppose the leftmost and rightmost such colliders are $k,k'$ (it is possible that $k=k'$), then $k,k'$ must be in $\Anc[S\cup\{w\}]\setminus\Anc[S]\subseteq\Anc[w]\subset\Anc[x]$. Consider the path $Q$ in the graph by cutting out the parts between $k,x$ (and $k',x$) on $P$ and replacing them with directed edges from $k$ to $x$ (and from $k'$ to $x$). Compared to $P$, the additional non-colliders on $Q$ are all on the directed path from $k$ to $x$ (or $k'$ to $x$). They are not in $S$ since $k,k'\notin \Anc[S]$, and thus $Q$ has no non-colliders in $S$. 
    
Compared to $P$, there is no collider on $P$ that is not in $\Anc[S]$ and is still on $Q$ by the fact that $k,k'$ are leftmost and rightmost colliders on $P$ that are not in $\Anc[S]$. Therefore, $x$ must be a collider on $Q$, or else $Q$ is active given $S$ and $u\not\CI v\mid S$. Therefore all non-colliders on $Q$ are not in $S\cup\{x\}$. Every collider on $Q$ is either $x$ or a collider of $P$, which is in $\Anc[S\cup\{x\}]$. Thus $Q$ is active given $S\cup\{x\}$. Therefore $u\not\CI v\mid S\cup \{x\}$.

Next we show that $D_S\cap \src(\bar{S})=\varnothing$: for contradiction assume that there exists a vertex $a \in \src(\bar{S})$ such that $a \not \in \bar{S} \backslash D_{S}$, that is $a \in \src(\bar{S})$ and for some vertex $u\in V, v \in \bar{S}$, $v \CI u|S \text{ and } v \not \CI u|S \cup \{a\}$.

Since $v \CI u|S \text{ and } v \not \CI u|S \cup \{a\}$, there exists a path $P$ between $v$ and $u$ which is inactive when conditioned on $S$ but is active upon conditioning on $S \cup \{ a\}$. Moreover, this path contains a vertex $b$ that is a collider on $P$ and satisfies: $a \in \Des[b]$ and $\Des[b] \cap S=\varnothing$. Since $\Des[b] \cap S=\varnothing$, we have that $b \in \bar{S}$. Furthermore, since $b \in \bar{S}, a \in \src(\bar{S})$ and $a \in \Des[b]$, this implies that $b=a$. Therefore, the path $P$ takes the form: $P=v \dots \to a \gets \dots u$. All the colliders on the path $P$ either belong to or have descendant in the set $S \cup \{ a\}$. 

Now consider the path $v \dots \to a$ and note that it is active given $S$. Let $k$ be the number of vertices between $v$ and $a$ on this path $v-v_1\dots v_k\to a$. It is immediate that $v_{k} \in S$ since $a \in \src(\bar{S})$. However, since $v_k \in S$, and since we condition on the set $S$, this should be a collider for the path $Q$ to be active, which is not possible. Thus we get a contradiction, which completes the proof.
\end{proof}

\subsubsection{Type-II Set $E_S$}

\begin{proof}[Proof of Lemma~\ref{lem:DS-ES-property} for $E_S$]
We first show that if $w\in E_S$, it must hold that $y\in E_S$ for any $y\in\Des(w)$: since $w\in E_S$, there is $u\in S$ and $v,v'\in \bar{S}\setminus D_S$, such that $u\CI v'\mid S\cup\{v\}$ and $u\not\CI v'\mid S\cup\{v,w\}$. We will show $u\not\CI v'\mid S\cup\{v,y\}$. 

Since $u\CI v'\mid S\cup\{v\}$ and $u\not\CI v'\mid S\cup\{v,w\}$, by Lemma~\ref{lemma:v-generalization} (note that the set ``$S$'' in the exposition of Lemma~\ref{lemma:v-generalization} can be an arbitrary subset), we know  that $v'\not\in\Des[w]$. Assume on the contrary that $u\CI v'\mid S\cup\{v,y\}$. Since $u\CI v'\mid S\cup\{v\}$ and $u\not\CI v'\mid S\cup\{v,w\}$, there is a path $P$ between $u,v'$ that is active given $S\cup\{v,w\}$ but inactive given $S\cup\{v\}$ or $S\cup\{v,y\}$. 
This means that (1) $y$ is a non-collider on $P$, (2) all colliders on $P$ are in $\Anc[S\cup\{v,w\}]$, (3) $P$ has a collider that is in $\Anc[w]\setminus\Anc[S\cup \{v\}]$.

Note that $v'\not\in\Des[y]$ since $y\in\Des[w]$ but $v'\not\in\Des[w]$.
In addition, we also have $u\not\in\Des[w]$ since $u$ is in the prefix vertex set $S$ but $w\in\bar{S}$. Therefore $y$ being a non-collider on $P$ means there is a collider $x$ on $P$ such that $y\in\Anc(x)$. Note that this collider has to be in $\Anc[S\cup\{v,w\}]$. However, since $y\not\in \Anc[S\cup\{w\}]$ and $y\in\Anc(x)$, it must hold that $x\in \Anc[v]$, which means $y\in\Anc(v)$. Since $y\in\Des(w)$, this means $v\in\Des(w)$, which makes $\Anc[w]\setminus\Anc[S\cup \{v\}]=\varnothing$. This violates (3) above, a contradiction.

Next we show that $E_S\cap \src(\bar{S})=\varnothing$: if there exists $w\in \Bar{S} \backslash D_{S}$ such that $w\in E_S\cap \src(\bar{S})$. Then $u \CI v'\mid S \cup \{v\} \text{ and } u \not \CI v' \mid S \cup \{v,w\}$ for some $u\in S$ and $v,v'\in\bar{S}\setminus D_S$. Thus there is a path $P$ connecting $u$ and $v'$ such that $P$ is active given $S\cup\{v,w\}$ but inactive given $S\cup\{v\}$. 

Therefore all non-colliders on $P$ are not in $S\cup\{v,w\}$, and there is a collider $x$ on $P$ such that $x\in\Anc[w]$. Note that since $w\in\src(\bar{S})$, it must hold that $x\in S$. Since $x$ is a collider, there is $y\neq u$ such that $y\to x$ on $P$. Since $S$ is a prefix and $x\in S$, we have $y\in S$. However, $y$ is a non-collider on $P$, which contradicts $P$ active given $S\cup\{v,w\}$ and completes the proof.
\end{proof}

\subsection{Proof of Lemma~\ref{lem:FS-property}}

\lemFSproperty*

\begin{proof}
We first show that if $w\in F_S$, then for any $y\in\Des(w)$, we have $y\in E_S\cup F_S$. Since $w\in F_S$, we have  $u\not\CI v\mid S, u\CI w\mid S\cup\{v\}$, and $v\not\CI w\mid S$ for some $u\in S$ and $v\in\bar{S}\setminus D_S$.

Since $v\not\CI w\mid S$, there is an active path between $v,w$ given $S$. Consider extending this path by the directed path from $w$ to $y$. Note that none of vertices on the directed path from $w$ to $y$ are in $S$, since $S$ is prefixed and $w\not\in S$. Therefore, this extended path is also active given $S$, which means $v\not\CI y\mid S$. 

Thus, if $y\not\in F_S$, then it must hold that $u\not\CI y\mid S\cup\{v\}$. This means there is an active path, denoted by $P$, between $u$ and $y$ given $S\cup\{v\}$. Consider extending this path by the directed path from $w$ to $y$, denoted as $Q$ (which exists in the graph). Compared to $P$, the additional non-colliders on $Q$ are not in $S\cup\{v\}$: for $S$, this is because $S$ is prefix, $w\not\in S$, and all additional non-colliders are descendants of $w$; for $v$, this is because $v\not\in\Des(w)$ (by Lemma~\ref{lem:general-meekrule-1}, $u\not\CI v\mid S$ and $u\CI w\mid S\cup\{v\}$). Thus $Q$ is active given $S\cup\{v\}$ unless $y$ is a collider on $Q$. Since $u\CI w\mid S\cup\{v\}$, the path $Q$ must be inactive given $S\cup\{v\}$, which means $y$ is a collider on $Q$. This means $Q$ is active given $S\cup\{v,y\}$. Therefore, $u\not\CI w\mid S\cup \{v,y\}$. Together with $u\CI w\mid S\cup \{v\}$, we have $y\in E_S$.

Next we show that $F_S\cap \src(\bar{S})=\varnothing$. Assume on the contrary that $w\in F_S\cap \src(\bar{S})$. Since $w\in F_S$, we have $u \not \CI v|S, v \not \ci w|S \text{ and } u \CI w|S \cup \{v\}$ for some $u\in S$ and $v\in\bar{S}\setminus D_S$. By Lemma~\ref{lem:general-meekrule-1}, we have $v\notin\Des[w]$. However, since $v\not\CI w\mid S$, there must be an active path $P$ between $v$ and $w$ given $S$. This path cannot have any vertex in $S$; otherwise consider the first vertex that is in $S$; since $v\not\in S$ and $S$ is prefix, such vertex must be a non-collider which would make $P$ inactive given $S$. Therefore $P$ is fully in $\bar{S}$. Since $w\in\src(\bar{S})$, we must have $P:v-\dots\leftarrow w$. However since $v\not\in \Des[w]$, there must be an edge $\rightarrow$ on $P$. This means that there must be a collider on $P$. Since this collider is not in $S$, it makes $P$ inactive given $S$, a contradiction.
\end{proof}

\subsection{Remarks}

The above proofs suffice as intermediate results to show Theorem~\ref{thm:prefix-subset}, which we proved in Section~\ref{sec:4-2}.

Regarding Lemma~\ref{lem:main-lem} in Section~\ref{sec:4-3}, we will prove it in Appendix~\ref{app:c} after proving Lemma~\ref{lem:one-node-anc-src}, since it depends on Lemma~\ref{lem:one-node-anc-src}.

\section{Missing Proofs of Causally Consistent Partition Graph Representations}\label{app:c}

\subsection{Proof of Lemma~\ref{lem:one-node-anc-src}}

To prove Lemma~\ref{lem:one-node-anc-src}, we will make use of the following lemma. 

\begin{lemma}\label{lem:prefix_subset_properties}
    Let $S\subseteq V$ be a prefix subset, then the following statements hold:
    \begin{itemize}
        \item $u \CI v~|~S \text{ for all }u,v \in \src(\bar{S})$.
        \item $S \cup U$ is a prefix subset for any $U \subseteq \src(\bar{S})$.
    \end{itemize}
\end{lemma}
\begin{proof}
    We first prove condition one. For contradiction, we assume that $u \not \ci v|S$, which implies that there exists an active path between $u$ and $v$ when conditioned on $S$. Let $P=u-u_1\dots u_k-v$ be the path and $u_1,\dots, u_k$ be the vertices along the path. Note that $k>0$ as $u$ and $v$ are not connected; because an edge between $u$ and $v$ would mean that one of these vertices is not a source node in $\bar{S}$. 

    Consider $u_1$ and note that there are two possibilities $u \to u_1$ or $u \gets u_1$. We start with the first case $u \to u_1$. Since $S$ is a prefix subset and $u \in \bar{S}$, $u \to u_1$ implies that $u_1 \in \bar{S}$. Since no vertex in $\bar{S}$ is conditioned upon, we have that the vertex $u_1$ is not a collider on the path $P$ and we have that the edge $u_1-u_2$ is directed as $u_1\to u_2$. Repeating the same argument for $u_2$ and all the other vertices in the path, we see that the path $P$ takes the form $P=u\to u_1\to \dots \to u_k\to v$. The previous argument implies that $v \in \Des[u]$ and therefore does not belong to $\src(\bar{S})$, which is a contradiction to our assumption that $u,v \in \src(\bar{S})$.

    Now consider the other case where $u \gets u_1$. Note that since $u \in \src(\bar{S})$, it is immediate that $u_1 \in S$. Furthermore, irrespective of the orientation between $u_1$ and vertex $u_2$, it holds that vertex $u_1$ is a non-collider on the path $P$. Moreover, since $P$ is an active path, $u_1$ should not be conditioned upon. However, since we condition on $S$ and as $u_1 \in S$, we have a contradiction. 

    In the above case analysis, we showed that there does not exist an active path between $u$ and $v$ when conditioned on $S$, which implies that $u \ci v|S$ and we conclude the proof for condition one.
    
    In the remainder of the proof, we focus our attention on condition two. Consider any subset $U \subseteq \src(\bar{S})$. For contradiction assume that $S \cup U$ is not a prefix subset, which implies that there exists a vertex $v \in V \backslash (S \cup U)$ such that $v \in \Anc[u]$ for some vertex $u \in S\cup U$. Since $S$ is a prefix subset, it is immediate that $u \not \in S$. Therefore, the only case is that $u \in U$. Note that both $u$ and $v$ belong to the set $\bar{S}$ and $v \in \Anc(u)$; both these expressions combined contradict the fact that $u \in \src(\bar{S})$. Therefore, it should be the case that $S \cup U$ is a prefix subset and we conclude the proof.
\end{proof}

Now we prove Lemma~\ref{lem:one-node-anc-src} restated below.

\lemonenodeancsrc*

\begin{proof}
We first show that for any $w\in\bar{S}\setminus D_S$, we have $|\Anc[w]\cap \src(\bar{S})|=1$. Since $w\in\bar{S}$, it must hold that $|\Anc[w] \cap \src(\bar{S})| \geq 1$. Assume now that there is $w\in\bar{S}$ such that $|\Anc[w] \cap \src(\bar{S})| \geq 2$. Let $v_1,v_2\in \Anc[w] \cap \src(\bar{S})$ such that $v_1\neq v_2$. By Lemma~\ref{lem:prefix_subset_properties}, we have $v_1\CI v_2\mid S$. Now consider the path $P$ by stitching together the two directed paths, one from $v_1$ to $w$ and another from $v_2$ to $w$. All non-colliders on $P$ are not in $S$ since $S$ is a prefix subset. The only collider on $P$ is $w$. Therefore $P$ is active given $S\cup\{w\}$. We have $v_1\not\CI v_2\mid S\cup\{w\}$, which means $w\in D_{v_1,S}\subset D_S$, a contradiction to $w\not\in D_S$.

Next we show that for any other $w'\in \bar{S}\setminus D_S$, we have $w\not\CI w'\mid S$ if and only if $\Anc[w]\cap \src(\bar{S})=\Anc[w']\cap \src(\bar{S})$. For the if direction, denote $s=\Anc[w]\cap \src(\bar{S})=\Anc[w']\cap \src(\bar{S})$ and consider the trek by joining the two directed paths from $s$ to $w$ and from $s$ to $w'$. Since this path has no colliders and it is fully in $\bar{S}$ (since $S$ is a prefix and $s\in\bar{S}$), it is active given $S$. Thus $w\not\CI w'\mid S$. For the only if direction, assume on the contrary that $w\not\CI w'\mid S$ but $\Anc[w]\cap \src(\bar{S})\neq\Anc[w']\cap \src(\bar{S})$. Let $P$ be the active path between $w,w'$ given $S$. Then all non-colliders on $P$ are in $\bar{S}$ and all colliders on $P$ are in $\Anc[S]=S$. This means that $P$ has no colliders; otherwise this collider is a child of the vertex next to it, which is a non-collider that is in $\bar{S}$. This means there is an edge from $\bar{S}$ to $S$, which is a contradiction with $S$ being prefix. Since $P$ has no colliders, it must satisfy $P\cap\Anc[w]\cap\Anc[w']\neq\varnothing$. However, since $\Anc[w]\cap \src(\bar{S})\neq\Anc[w']\cap \src(\bar{S})$, $|\Anc[w]\cap \src(\bar{S})|=1$ and $|\Anc[w']\cap \src(\bar{S})|=1$, $P\cap\Anc[w]\cap\Anc[w']\neq\varnothing\in \bar{S}$. This means there is a non-collider on $P$ that is in $S$, which would make it inactive given $S$, a contradiction.
\end{proof}

\subsection{Proof of Lemma~\ref{lem:main-lem}}

To prove Lemma~\ref{lem:main-lem}, we will make use of the following results.

\begin{lemma}\label{lem:one-anc-src-blocks-almost-all}
    Let $S$ be a prefix subset and $w\in \bar{S}\setminus (D_S\cup \src(\bar{S}))$. By Lemma~\ref{lem:one-node-anc-src}, let $\Anc[w]\cap \src(\bar{S})=\{v\}$. Then for any $u\in S$, if there is no directed path from $u$ to $w$ that does not intersect $\src(\bar{S})$, then $$u\CI w\mid S\cup\{v\}~.$$
\end{lemma}

\begin{proof}
    Assume on the contrary that $u\not\CI w\mid S\cup \{v\}$. Let $P$ be an active path from $u$ to $w$ conditioned on $v$. If $P\cap S\neq \{u\}$, then consider the last node on $P$ that is in $S$. Since $w\not\in S$, this node must be pointing into a node in $\bar{S}$ on $P$, which makes this node a non-collider. However, this node belongs to $S$, which means $P$ is inactive given $S\cup\{v\}$, and thus $P\cap S=\{u\}$. 

    There is also no collider on $P$. Otherwise consider the first collider; it will be in $\bar{S}$ since $P\cap S=\{u\}$. However, since $P$ is active, it will be in $\Anc[S\cup\{v\}]$. This is impossible since $v\in\src(\bar{S})$ and from Lemma~\ref{lem:prefix_subset_properties}, we know that $S\cup \{v\}$ is prefixed. Therefore, $P$ must be a directed path from $u$ to $w$, where the node $x$ adjacent to $u$ is in $\bar{S}$. This means that $x\in \Anc[w]$. If $x\not\in \src(\bar{S})$, then $P$ is a directed path from $u$ to $v$ that does not intersect $\src(\bar{S})$. If $x\in \src(\bar{S})$, since $\Anc[w]\cap \src(\bar{S})=\{v\}$, then it must hold that $x=v$. This means $P$ is inactive given $S\cup\{v\}$, a contradiction.
\end{proof}

\begin{lemma}\label{lem:ind-or-des}
Let $S$ be a prefix subset and $v\in\src(\bar{S})$. Then for any $w\in \bar{S}$, either $v\CI w\mid S$ or $w\in\Des[v]$.
\end{lemma}

\begin{proof}
    Suppose $w\not\in\Des[v]$; we will show that $v\CI w\mid S$. Assume on the contrary that $v\not\CI w\mid S$. Let $P$ be the active path between $v$ and $w$ given $S$. If $P$ intersects with $S$, then consider the last vertex on $P$ that is in $S$. This vertex must be a non-collider since $w\not\in S$. This contradicts $P$ being active given $S$. Therefore, $P$ is fully in $\bar{S}$. Since $v\in\src(\bar{S})$, we have $P:v\rightarrow \dots w$. Since $w\not\in \Des[v]$, there must be a collider on $P$. However, this collider is not in $S$ and $S$ is prefix, which means that $P$ is active given $S$, a contradiction.
\end{proof}

We now prove Lemma~\ref{lem:main-lem}, restated below. 

\lemmainlem*

\begin{proof}[Proof of Lemma~\ref{lem:main-lem}]
Assume $w \not \in \src(\bar{S})$. For contradiction, assume that \textcolor{black}{there is no covered edge from $\Anc[w]\cap\src(\bar{S})$ to $\Anc[w]$} and $w\in\bar{S}\setminus (D_S\cup E_S\cup F_S)$. 

Since $w\not\in\src(\bar{S})$, there is $v\in\src(\bar{S})\cap \Anc[w]$. As $v\in\Anc[w]$, there is a directed path from $v$ to $w$. Consider the longest directed path $P$ from $v$ to $w$. Let $v'$ be the adjacent vertex to $v$ on $P$, i.e., $P:v\rightarrow v'\dots \rightarrow w$. By Lemma~\ref{lem:DS-ES-property} and $w\not\in D_S$, we must have $v'\not\in D_S$. 
Since $v\rightarrow v'$ is \textcolor{black}{not a covered edge}, there could only be two cases:
\begin{itemize}
    \item There is $u\rightarrow v$ such that $u\not\in\Pa(v')$. Since $v\in \src(\bar{S})$, we must have $u\in S$. Note that $u\not\CI v\mid S$ and $v\not\CI w\mid S$ and the second condition of the lemma is not met. We must have $u\not\CI w\mid S\cup \{v\}$. By Lemma~\ref{lem:one-anc-src-blocks-almost-all}, there must exist a directed path $Q$ from $u$ to $w$ that does not intersect $\src(\bar{S})$. Consider the path between $u$ and $w'$ by joining $Q$ and the directed path from $v'$ to $w$. Since this path does not intersect with $S\cup\src(\bar{S})$ and $w$ is the only collider on it, we know that this path is active given $S\cup\{v\}\cup\{w\}$. Thus $u\not\CI v'\mid S\cup\{v\}\cup\{w\}$. Since the second condition of the lemma is not met, we must have $u\not\CI v'\mid S\cup\{v\}$. By Lemma~\ref{lem:one-anc-src-blocks-almost-all}, there must exist a directed path from $u$ to $v'$ that does not intersect $\src(\bar{S})$. Since $u\not\in\Pa(v')$, this path has at least length two. Let $k$ be the vertex on this path such that $k\to v'$. Note that $v\to v'\leftarrow k$. Therefore $v\not\CI k\mid S\cup \{v'\}$. Since $v'\not\in D_S$, we must have $v\not\CI k\mid S$. By Lemma~\ref{lem:ind-or-des}, we must have $k\in\Des[v]$. Thus we can increase the length of the directed path $P$ by replacing $v\to v'$ with $v\to\dots\to k\rightarrow v'$. This contradicts $P$ being the longest directed path from $v$ to $w$.
    \item There is $k\rightarrow v'$ such that $k\not\in\Pa(v)$. Note that $v\to v'\leftarrow k$. Therefore $v\not\CI k\mid S\cup \{v'\}$. Since $v'\not\in D_S$, we must have $v\not\CI k\mid S$. By Lemma~\ref{lem:ind-or-des}, we must have $k\in\Des[v]$. Thus we can increase the length of the directed path $P$ by replacing $v\to v'$ with $v\to\dots\to k\rightarrow v'$. This contradicts $P$ being the longest directed path from $v$ to $w$.
\end{itemize}
This completes the proof.
\end{proof}

\subsection{Remarks}

The above proofs suffice as intermediate results to show Theorem~\ref{thm:main-1}, which we proved in Section~\ref{sec:5}.

\section{Missing Proofs of Interventions}\label{app:d}

\subsection{Proof of Lemma~\ref{lem:6-1}}

\lemsixone*

\begin{proof}
For each $v\in I$, denote $\Des(v,I)$ as the set of $u\in I$ such that $u\not\CI_I v$. We first show that $\Des(v,I)$ is equal to the set of descendants of $v$ such that there exists a directed path from $v$  which is not cut by $I$.

Let $u\notin I$ such that $u \ci_I v $. If $u\in \Des(v,I)$, then there exists a directed path from $v$ to $u$ in the modified DAG $\cG^I$, which means it is active given $\varnothing$, a contradiction. Therefore the only if direction is proven.

For the if direction, let $u\notin I$ such that $u \not\ci_I v $. Then there is an active path $P:u-\dots-v$ in the modified DAG $\cG^I$. Since $P$ is active given $\varnothing$, there is no collider on $P$. Furthermore, since all incoming edges to any vertex in $I$ are removed in the modified DAG, this path must be $P:u\leftarrow\dots\leftarrow v$ and satisfies $P\cap I=\{v\}$. Thus $u\in\Des(v,I)$.

Next we show that $\cup_{v\in I\setminus S}\Des(v,I)=\Des(I\setminus S)\setminus (I\setminus S)$. This will prove the lemma. 

Since $\Des(v,I)$ is equal to the set of descendants of $v$ such that there exists a directed path from $v$ which is not cut by $I$, it is clear that $\cup_{v\in I\setminus S}\Des(v,I)\subseteq\Des(I\setminus S)\setminus (I\setminus S)$. Now let $u\in \Des(I\setminus S)\setminus (I\setminus S)$. Then there is $v\in I\setminus S$ such that $u\in \Des(v)$. Consider the directed path from $v$ to $u$ and let $v'$ be the last vertex on this path that is in $I$. Then $u\in\Des(v',I)$. By the fact that $S$ is prefix, we have from $v\not\in S$ that $v'\not\in S$. Thus $u\in \Des(v',I)\subset \cup_{v\in I\setminus S}\Des(v,I)$ and $\cup_{v\in I\setminus S}\Des(v,I)\supseteq\Des(I\setminus S)\setminus (I\setminus S)$. We therefore have $\cup_{v\in I\setminus S}\Des(v,I)=\Des(I\setminus S)\setminus (I\setminus S)$.
\end{proof}

\subsection{Proof Lemma~\ref{lem:int-helper}}

\leminthelper*

\begin{proof}
    On one hand, let $u\in\Pa(v)\setminus(S\cup \Des[I\setminus S])$. Clearly $u\not\in S\cup \Des[I\setminus S]$ and $u\not\CI v\mid V\setminus \Des[I\setminus S]$. Thus $u\in H_S^I(v)$, which proves $\Pa(v)\setminus(S\cup \Des[I\setminus S])\subseteq H_S^I(v)$.

    On the other hand, let $u\in H_S^I(v)$. Since $u\not\CI v\mid V\setminus \Des[I\setminus S]$, let $P:u-\dots-v$ be the active path given $V\setminus \Des[I\setminus S]$. Then all non-colliders on $P$ are in $\Des[I\setminus S]$, and all colliders on $P$ are in $\Anc[V\setminus \Des[I\setminus S]]$

    Note that $\Des[I\setminus S]\cap \Anc[V\setminus \Des[I\setminus S]]=\varnothing$. Otherwise let $w\in\Des[I\setminus S]\cap \Anc[V\setminus \Des[I\setminus S]]$. Since $w\in \Anc[V\setminus \Des[I\setminus S]]$, there is $w'\not\in\Des[I\setminus S]$ such that $w'\in\Des[w]$. However $w\in\Des[I\setminus S]$, which means $w'\in\Des[I\setminus S]$. A contradiction.

    If $P$ has a collider $y\to x\leftarrow z$, then $z$ is either a non-collider or $v$. Either case we have $z\in\Des[I\setminus S]$. This means $x\in\Des[I\setminus S]$. However, $x\in S\cup \Anc[V\setminus \Des[I\setminus S]]$ and $\Des[I\setminus S]\cap\Anc[V\setminus \Des[I\setminus S]]=\varnothing$. A contradiction. Thus there is no collider on $P$.

    Denote the vertex next to $u$ on $P$ as $t$, i.e., $P:u-t-\dots -v$. Then $t$ is either a non-collider or $v$. Either case $t\in\Des[I\setminus S]$. Therefore it must be $u\to t$, otherwise $u\in\Des[I\setminus S]$. Furthermore, as $P$ has no colliders, it must be $P:u\to t \dots\to v$. Thus $u\in\Anc(v)$. Together with $u\not\in S\cup \Des[I\setminus S]$, we have $u\in\Anc(v)\setminus(S\cup \Des[I\setminus S])$. This proves $H_S^I(v)\subseteq \Anc(v)\setminus (S\cup \Des[I\setminus S])$. 
\end{proof}

\subsection{Proof of Lemma~\ref{lem:J-S-I}}

\lemJSI*

\begin{proof}
    Assume on the contrary that $v\in J_S^I$ but $w\in \Des(v)$ such $w\not\in J_S^I$. Since $v\in J_S^I$, there could be two cases:
    \begin{itemize}
        \item $v\in\Des(I\setminus S)$. Then clearly $w\in\Des(v)\subseteq\Des(I\setminus S)$. A contradiction.   
        \item $v\in I\setminus S$ and there is $u\in H_S^I(v)\cap \bar{S}$. Then $w\in\Des(v)\subset \Des(I\setminus S)$. A contradiction.
    \end{itemize}
    Next we show that $\src(\bar{S})\cap J_S^I\neq \varnothing$. Note that by $\src(\bar{S})\cap \Des(\bar{S})=\varnothing$ and $\Des(I\setminus S)\subseteq \Des(\bar{S})$, we have $\src(\bar{S})\cap \Des(I\setminus S)=\varnothing$. In addition, for any $v\in \src(\bar{S})\cap (I\setminus S)$, by Lemma~\ref{lem:int-helper}, we have $H_S^I(v)\subset \Anc(v)\setminus S=\varnothing$. Thus $\src(\bar{S})\cap \{v\in I\setminus S:~ H_S^I(v)\cap\bar{S}\neq \varnothing\}=\varnothing$. We then have $\src(\bar{S})\cap J_S^I=\varnothing$.
\end{proof}

\subsection{Remarks}

The above proofs suffice as intermediate results for proving Theorem~\ref{thm:prefix-subset-int}. Then together with Lemma~\ref{lem:int-covered-sep} (proven in Section~\ref{sec:6-1}), we can prove Theorem~\ref{thm:main-2}, which is given in Section~\ref{sec:6-2}.

\section{Details of Numerical Experiments}\label{app:e}

\textbf{Implementation Details.} For \texttt{FCI} and \texttt{RFCI}, we used the implementations in \citet{kalisch2024package}, which is written in R with C++ accelerations. For \texttt{PC} and \texttt{GSP}, we used the implementation in \citet{squirescausaldag}, which us written in python. Our method, \texttt{CCPG}, is written in python. The acceleration of R (with C++) can be viewed by comparing two implementations of PC (stable) in \Cref{fig:runtime}.

\textbf{Remark on causal sufficiency.} Among the constraint-based methods, we marked the ones that do not assume causal sufficiency in \Cref{fig:nsample}. These methods run additional tests to check for unobserved causal variables, and therefore might require more samples compared to, e.g., \texttt{PC}, since the underlying system we test on satisfies causal sufficiency.